%% file: m1031.tex
\documentclass{ecai} 

\usepackage{latexsym}
\usepackage{amssymb}
\usepackage{amsmath}
\usepackage{amsthm}
\usepackage{booktabs}
\usepackage{enumitem}
\usepackage{graphicx}
\usepackage{color}

\newtheorem{theorem}{Theorem}
\newtheorem{proposition}{Proposition}

\usepackage{lipsum}
\usepackage{amssymb}
\usepackage{algorithm}
\usepackage{algorithmic}
\usepackage{subcaption}
\usepackage{pifont}
\usepackage{amsthm}
\usepackage[export]{adjustbox}%
\usepackage{soulutf8}

\usepackage{caption}
\usepackage{booktabs}
\usepackage{multirow}
\usepackage{soul}
\usepackage{graphicx}
\usepackage{wrapfig}
\usepackage{placeins}
\AddToHook{cmd/appendix/before}{%
  \setcounter{equation}{0}%
  \setcounter{theorem}{0}%
  \setcounter{proposition}{0}%
}
\usepackage{xcolor}
\usepackage[most]{tcolorbox}
\newtcolorbox{highlighted}{colback=yellow,coltext=red,breakable}
{\colorbox{yellow}}%
{}
\usepackage{pgfplots}

\usepackage{natbib}
\input{math_commands.tex}

\usepackage{color, colortbl}
\usepackage{float}
\usepackage{hyperref}
\usepackage{cleveref}
\usepackage[symbol]{footmisc}

\newcommand{\subfigspace}{\vspace{10px}}
\newcommand{\figspace}{\vspace{8px}}

\definecolor{MSP}{rgb}{0.804, 0.804, 0.922}
\definecolor{Energy}{rgb}{0.706, 0.78, 0.741}
\definecolor{Entropy}{rgb}{0.859, 0.753, 0.753}
\definecolor{MLS}{rgb}{0.9,0.9,0.9}
\definecolor{MCM}{rgb}{0.78, 0.78, 0.71}
\definecolor{mixdiff_bar}{rgb}{0.612, 0.078, 0.078}

\newcommand{\BibTeX}{B\kern-.05em{\sc i\kern-.025em b}\kern-.08em\TeX}

\begin{document}


\begin{frontmatter}


\paperid{1031} 


\title{Perturb-and-Compare Approach for Detecting Out-of-Distribution Samples in Constrained Access Environments}


\author[A]{\fnms{Heeyoung}~\snm{Lee}\footnote[1]{Equal contribution.}}
\author[B]{\fnms{Hoyoon}~\snm{Byun}\footnote[1]{}}
\author[C]{\fnms{Changdae}~\snm{Oh}} 
\author[A]{\fnms{JinYeong}~\snm{Bak}\footnote[*]{}}
\author[B]{\fnms{Kyungwoo}~\snm{Song}\footnote[*]{Corresponding Authors: jy.bak@skku.edu, kyungwoo.song@yonsei.ac.kr.}}

\address[A]{Sungkyunkwan University, Suwon, South Korea}
\address[B]{Yonsei University, Seoul, South Korea}
\address[C]{University of Wisconsin--Madison, Madison, Wisconsin, United States}


\begin{abstract}
Accessing machine learning models through remote APIs has been gaining prevalence following the recent trend of scaling up model parameters for increased performance. Even though these models exhibit remarkable ability, detecting out-of-distribution (OOD) samples remains a crucial safety concern for end users as these samples may induce unreliable outputs from the model. In this work, we propose an OOD detection framework, MixDiff, that is applicable even when the model's parameters or its activations are not accessible to the end user. To bypass the access restriction, MixDiff applies an identical input-level perturbation to a given target sample and a similar in-distribution (ID) sample, then compares the relative difference in the model outputs of these two samples. MixDiff is model-agnostic and compatible with existing output-based OOD detection methods. We provide theoretical analysis to illustrate MixDiff's effectiveness in discerning OOD samples that induce overconfident outputs from the model and empirically demonstrate that MixDiff consistently enhances the OOD detection performance on various datasets in vision and text domains.

\end{abstract}

\end{frontmatter}





\section{Introduction}\label{sec:intro}

Recent developments in deep neural networks (DNNs) opened the floodgates for a wide adaptation of machine learning methods in various domains such as computer vision, natural language processing and speech recognition. As these models garner more users and widen their application area, the magnitude of impact that they may bring about when encountered with a failure mode is also amplified. One of the causes of these failure modes is when an out-of-distribution (OOD) sample is fed to the model. These samples are problematic because DNNs often produce unreliable outputs if there is a large deviation from the in-distribution (ID) samples that the model has been validated to perform well. 

OOD detection is the task of determining whether an input sample is from ID or OOD. This work focuses on semantic shift \citep{yang2021oodsurvey} where distribution shift is manifested by samples of unseen class labels at test time. Several studies explore measuring how uncertain a model is about a target sample relying on the model’s output \citep{hendrycks2017a, xixi2023GEN}. While these methods are desirable in that they do not assume access to the information inside the model, they can be further enhanced given access to the model’s internal activations, \citep{sun2021react} or its parameters \citep{9156473}. However, the access to the model’s internal states is not always permitted.  With the advent of foundation models \citep{radford2021learning, openai2023gpt4}, users often find themselves interacting with the model through remote APIs \citep{sun2022bbt}. This limits the utilization of rich information inside the model \citep{NEURIPS2021_063e26c6}, as well as the modification possibilities \citep{sun2022dice} that can be effectively used to detect OOD samples. In this work, we explore ways to bypass this access restriction through the only available modification point, namely, the models' inputs.

Data samples in the real world may contain distracting features that can negatively affect the model's performance. Sometimes these distractors may possess characteristics resembling a class that is different from the sample's true label. In this case, the model's predictions for an ID sample could become uncertain as it struggles to decide which class the sample belongs to. Similarly, the model could put too much emphasis on a feature that resembles a certain in-distribution characteristic from an OOD sample, outputting an overconfident prediction, even though the sample does not belong to any of the classes that the model was tasked to classify.

\begin{figure*}[t]
  \centering
  \begin{subfigure}[b]{0.50\linewidth}
    \centering
        \includegraphics[width=1\linewidth]{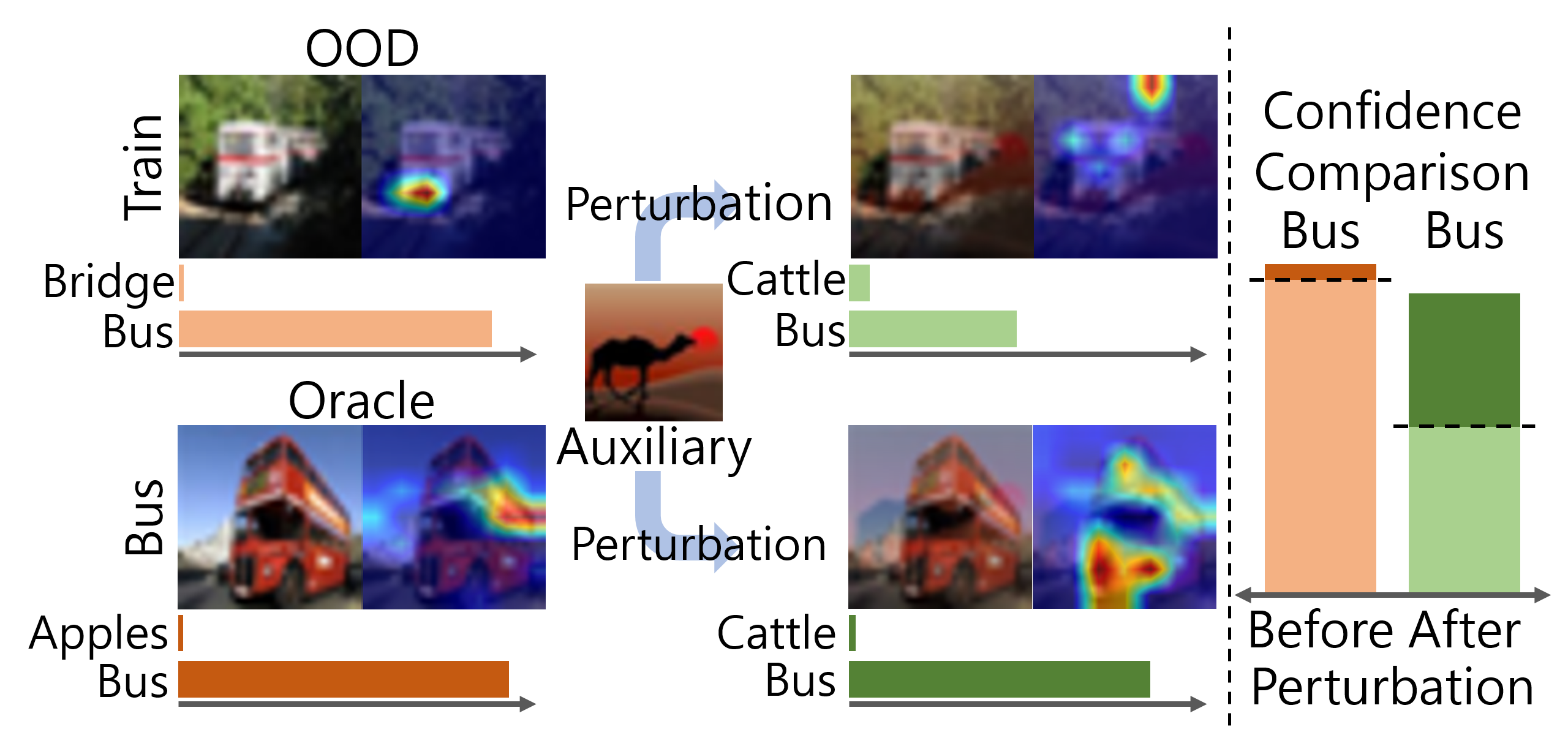}
    \caption{} 
    \label{fig:illust}
  \end{subfigure}
  \begin{subfigure}[b]{0.50\linewidth}
    \centering
            \includegraphics[width=0.7\textwidth]{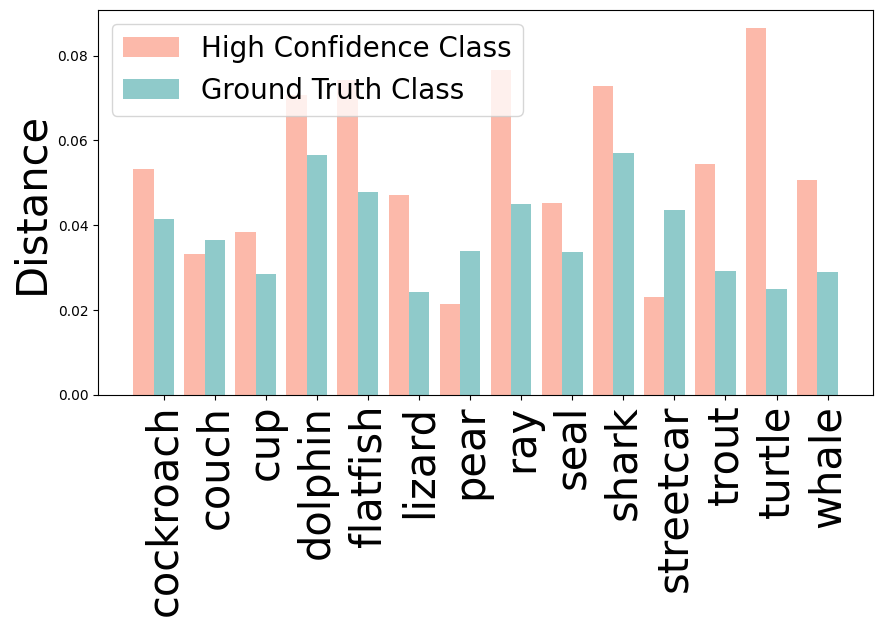}
    \caption{} 
    \label{fig:cam_dist}
  \end{subfigure}
  \subfigspace
    \caption{\textbf{(a)} Class activation map of an OOD sample (train) for the predicted class (bus) exhibits a high degree of sensitivity when an auxiliary image (camel) is mixed to it. The same class activation map of an image of an actual bus is more robust to the same perturbation. (Top 2 classes are shown). \textbf{(b)} Average $L_1$ distance of the class activation maps of high confidence class and the ground truth class after perturbation (averaged over each OOD class).}
  \figspace
\end{figure*}

We start from the intuition that the contributing features in a misclassified sample, either misclassified as ID or OOD, will tend to be more sensitive to perturbations. In other words, these features that the model has overemphasized will be more brittle when compared to the actual characteristics of the class that these features resemble. Take as an example the image that is at the top left corner of Figure \ref{fig:illust}. This sample is predicted to be a bus with a high confidence score, despite it belonging to an OOD class train. When we exact a perturbation to this sample by mixing it with some other auxiliary sample, the contribution of the regions that led to the model’s initial prediction is significantly reduced as can be seen by the change in the class activation maps (CAM) \citep{Chen_2022_ACCV}. However, when the same perturbation is applied to an actual image of a bus, the change is significantly less abrupt. The model's prediction scores show a similar behavior. 

To experimentally verify the intuition, we collect OOD samples that induce high confidence scores from the model and compute CAMs for these samples before and after perturbation. Two versions of CAMs are computed with a zero-shot image classifier using CLIP model \cite{radford2021learning}. One with respect to the predicted class of the sample and the other with respect to the ground truth class of the sample. Figure \ref{fig:cam_dist} shows that the $L_1$ distance between the CAMs of the unperturbed and perturbed versions of an OOD sample's predicted class tends to be higher when compared to its ground truth class, even though the OOD sample had a high confidence score for that class. We provide experimental details in Appendix E.

Motivated by the above idea, we propose an OOD detection framework, MixDiff, that exploits this perturb-and-compare approach without any additional training. MixDiff employs a widely used data augmentation method Mixup \citep{zhang2018mixup} as the perturbation method so as to promote diverse interaction of features in the samples. Its overall procedure is outlined as follows: (1) perturb the target sample by applying Mixup with an auxiliary sample and get the model's prediction by feeding the perturbed target sample to the model; (2) perturb an ID sample of the predicted class of the target (oracle sample in Figure \ref{fig:illust}) by following the same procedure; (3) compare the uncertainty scores of the perturbed samples. By comparing how the model's outputs of the target sample and a similar ID sample behave under the same perturbation, MixDiff augments the limited information contained in the model's prediction scores. This gives MixDiff the ability to better discriminate OOD and ID samples, even when the model's prediction scores for the original samples are almost identical.

We summarize our key contributions and findings as follows: (1) We propose an OOD detection framework, MixDiff, that enhances existing OOD scores in constrained access environments where only the models' inputs and outputs are accessible. (2) We provide a theoretical insight as to how MixDiff can mitigate the overconfidence issue of existing output-based OOD scoring functions. (3) MixDiff consistently improves various output-based OOD scoring functions when evaluated on OOD detection benchmark datasets in constrained access scenarios where existing methods' applicability is limited.

\section{Related work}
\paragraph{Output-based OOD scoring functions}
Various works propose OOD scoring functions measuring a classifier's uncertainty from its prediction scores. Some of these methods rely solely on the model's prediction probability. Maximum softmax probability (MSP) \citep{hendrycks2017a} utilizes the maximum value of the prediction distribution. \citet{Thulasidasan2021AnEB} use Shannon entropy as a measure of uncertainty, while GEN \citep{xixi2023GEN} proposes a generalized version of the entropy score. KL Matching \citep{Hendrycks2022ScalingOD} finds the minimum KL divergence between the target and ID samples. D2U \citep{yilmaz-toraman-2022-d2u} measures the deviation of output distribution from the uniform distribution. If we take a step down to the logit space, maximum logit score (MLS) \citep{Hendrycks2022ScalingOD} utilizes the maximum value of the logits. Energy score \citep{liu2020energy} takes $\text{LogSumExp}$ over the logits for the OOD score. MCM \citep{mingdelving} emphasizes the importance of temperature scaling in vision-language models \citep{radford2021learning}. While these output-based methods are desirable in that they take a relaxed assumption on model accessibility, they suffer from the model's overconfidence issue \citep{nguyen2015deep}. 
This motivates us to investigate the perturb-and-compare approach as a calibration measure.

\paragraph{Enhancing output-based OOD scores}
Another line of work focuses on enhancing the aforementioned output-based OOD scores to make them more discriminative. ODIN \citep{liang2018enhancing} 
utilizes Softmax temperature scaling and gradient-based input preprocessing to enhance MSP \citep{hendrycks2017a}. ReAct \citep{sun2021react} alleviates the overconfidence issue by clipping the model's activations if they are over a certain threshold. BAT \citep{zhu2022boosting} uses batch normalization \citep{pmlr-v37-ioffe15} statistics for activation clipping. DICE \citep{sun2022dice} leverages weight sparsification to mitigate the over-parameterization issue. Recently, methods that are based on activation \citep{djurisic2023extremely} or weight pruning  \citep{Ahn_2023_CVPR} approaches also have been proposed. These approaches effectively mitigate the overconfidence issue. However, all of these methods require access to either gradients, activations or parameters; hence limits their applicability in remote API environments. Our work stands out as an OOD score enhancement method in constrained access environments, where models' gradients, activations, and parameters are not accessible, leaving the model inputs as the only available modification point.

\paragraph{Utilization of deeper access for more discriminative OOD scores}
Several studies exploit the rich information that the feature space provides when designing OOD scores. \citet{Olber2023, zhang2023outofdistribution} utilize ID samples' activations for comparison with a target sample. Models' inner representations are employed in methods that rely on class-conditional Mahalanobis distance \citep{NEURIPS2018_abdeb6f5}.
ViM \citep{Wang_2022_CVPR} proposes an OOD score that complements the energy score \citep{liu2020energy} with additional information from the feature space. \citet{sun2022knnood} use the target sample's feature level KNN distance to ID samples. GradNorm \citep{NEURIPS2021_063e26c6} employs the gradient of the prediction probabilities' KL divergence to the uniform distribution. \citet{Zhang_2023_CVPR} show that decoupling MLS \citep{Hendrycks2022ScalingOD} can lead to increased detection performance if given access to the model parameters. However, these methods are not applicable to black-box API models where one can only access the model's two endpoints, namely, the inputs and outputs.

\begin{figure*}[t]
\centering
\includegraphics[width=0.9\textwidth]{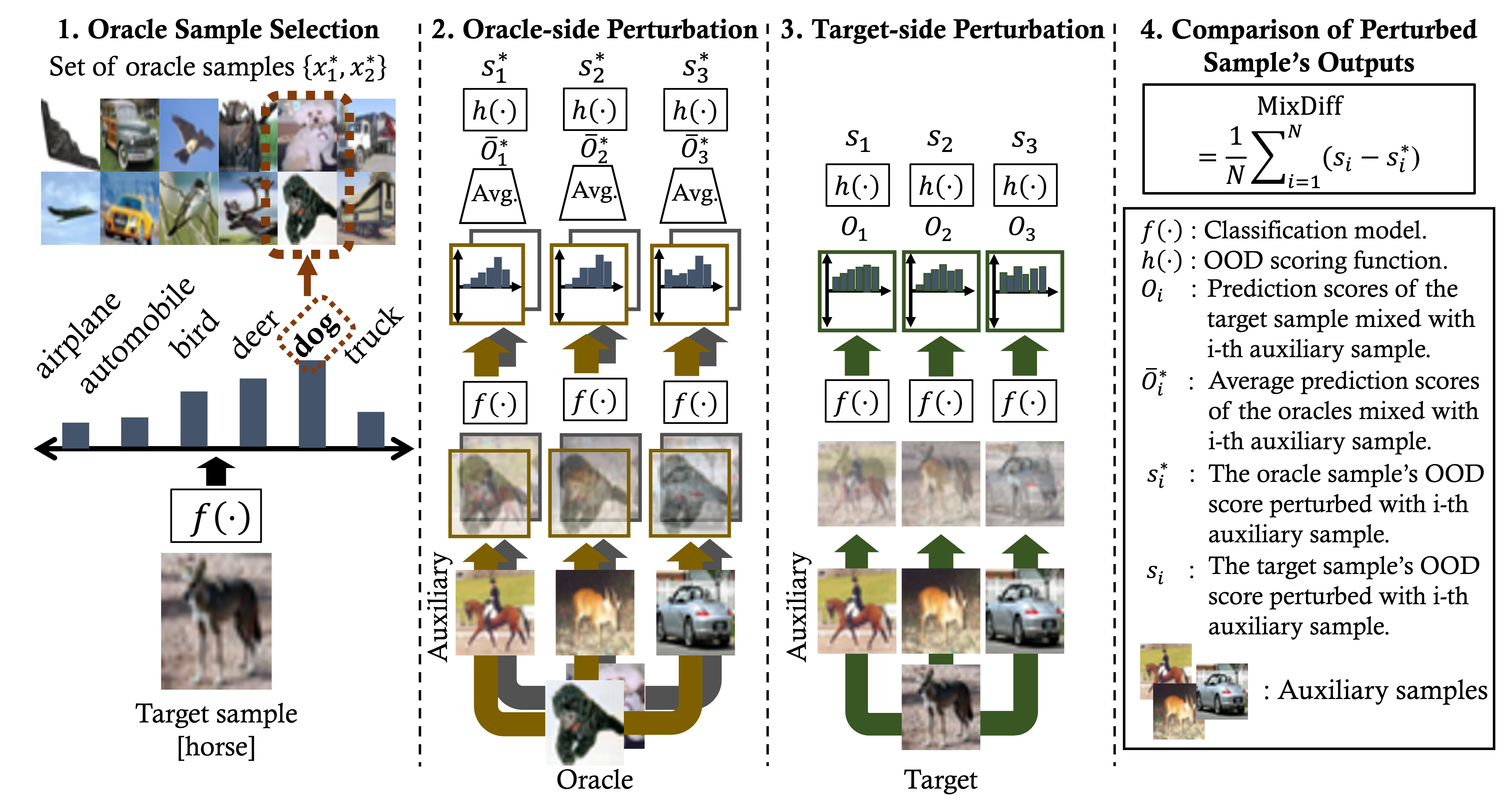}
\caption{The overall figure of MixDiff with the number of Mixup ratios, $R=1$, the number of classes, $K=6$, the number of auxiliary  samples, $N=3$, and the number of oracle instances, $M=2$. We omit Mixup ratio subscript $r$ for simplicity.
}
\label{fig:overall}
\figspace
\end{figure*}

\section{Methodology}\label{sec:method}
In this section, we describe the working mechanism of MixDiff framework. MixDiff is comprised of the following three procedures: (1) find ID samples that are similar to the target sample and perturb these samples by performing Mixup with an auxiliary sample; (2) perturb the target sample by performing Mixup with the same auxiliary sample; (3) measure the model's uncertainty of the perturbed target sample \textit{relative} to the perturbed ID samples. We now provide a detailed description of each procedure.

\paragraph{Oracle-side perturbation}

We feed the given target sample, $x_t$, to a classification model $f(\cdot)$ and get its prediction scores for $K$ classes, $O_t$, and the predicted class label, $\widehat{y_t}$, as shown in Equation \ref{eq:y_pred}.

\begin{align} \label{eq:y_pred}
O_t &= f(x_t) \in \mathbb{R}^K \text{, }\;\; \widehat{y_t} = \argmax (O_t)
\end{align}

Next, we assume a small set of $M$ labeled samples, $\Omega_k = \{(x^{*}_{m},y^{*}_{k})\}_{m=1}^{M}$, for each class label $k$. We refer to these samples as the oracle samples. From these, we take the samples that are of the same label as the predicted label $\widehat{y}_{t}$. Then, we perturb each oracle sample, $x_m^*$, by performing Mixup with an auxiliary sample, $x_i \in \{ x_i \}_{i=1}^{N}$, with Mixup rate $\lambda_{r}$.

\begin{equation}\label{eq:oracle_mixup}
x^{*}_{mir} = \lambda_{r} x^{*}_{m} + (1-\lambda_{r})x_i,  \text{ where } y^{*}_{k} = \widehat{y_t} 
\end{equation}

We feed the perturbed oracle sample to the classification model $f(\cdot)$ and get the model's prediction scores, $O^{*}_{mir} = f(x^{*}_{mir}) \in \mathbb{R}^K$. Then, we average the perturbed oracle samples' model outputs, to get $\Bar{O}^{*}_{ir} = \frac{1}{M}\sum_{m=1}^{M}O^{*}_{mir}$. Finally, we compute the perturbed oracle samples' OOD score, $s^{*}_{ir} \in \mathbb{R}$, with an arbitrary output-based OOD scoring function $h(\cdot)$ such as MSP or MLS, \textit{i.e.}, $s^{*}_{ir} = h\left(\Bar{O}^{*}_{ir}\right) \in \mathbb{R}$.

\paragraph{Target-side perturbation}
We perturb the target sample $x_t$ with the same auxiliary samples $\{x_i\}_{i=1}^N$, as $x_{ir} = \lambda_{r} x_{t} + (1-\lambda_{r})x_i$, and compute the OOD scores of the perturbed target sample as follows: 
\begin{equation}
    O_{ir} = f(x_{ir}) \in \mathbb{R}^K, \;\; s_{ir} = h(O_{ir}) \in \mathbb{R}
\end{equation}
\paragraph{Comparison of perturbed samples' outputs}
From the perturbed target's and oracles' uncertainty scores, ($s^{*}_{ir}$, $s_{ir}$), we calculate the MixDiff score for the target sample, $x_t$, as shown in Equation \ref{eq:mixdiff}. It measures the model's uncertainty score of the target sample relative to similar ID samples when both undergo the same Mixup operation with an auxiliary sample $x_i$, then takes the average of the differences over the auxiliary samples and the Mixup ratios. 
We provide descriptions and illustrations of the overall procedure in Algorithm \ref{alg:mixdiff} and Figure \ref{fig:overall}.

\begin{align}
\text{MixDiff} = \frac{1}{RN}\sum_{r=1}^{R}\sum_{i=1}^{N}(s_{ir} - s^{*}_{ir}) \label{eq:mixdiff}
\end{align}

We calibrate the base OOD score for the target sample, $h(f(x_t))$, by adding the MixDiff score with a scaling hyperparameter $\gamma$ to it so as to mitigate the model's over- or underconfidence issue.

\begin{algorithm}[h]
\caption{Computation of MixDiff Score}
\label{alg:mixdiff}
\textbf{Require}: target sample $x_t$, set of auxiliary samples $\{x_i\}_{i=1}^{N}$, set of Mixup rates $\{\lambda_r\}_{r=1}^R$, set of oracle samples for all $K$ classes $\{\Omega_k\}_{k=1}^{K}$ where $\Omega_k = \{ (x_m^*, y^*_k) \}_{m=1}^M$, classifier model $f(\cdot)$, OOD scoring function $h(\cdot)$
\begin{algorithmic}[1]
\STATE $O_t = f(x_t)$
\STATE $\widehat{y_t} = \argmax(O_t)$
\STATE $\{(x_{m}^{*}, y_{k}^{*})\}_{m=1}^{M} \gets \Omega_k, \text{where \;} y_{k}^{*}=\widehat {y_t}$
\FOR{$i \in \{1, \dots, N \} $}
\FOR{$r \in \{1, \dots, R \}$}
\FOR{$ m \in \{ 1, \dots, M \} $}
\STATE $O^*_{mir} \gets f(\lambda_{r} x^{*}_{m} + (1-\lambda_{r})x_i)$
\ENDFOR
\STATE $s^*_{ir} \gets h\left(\frac{1}{M}\sum_{m=1}^{M}O^{*}_{mir} \right)$
\STATE $O_{ir} \gets f(\lambda_{r} x_{t} + (1-\lambda_{r})x_i)$
\STATE $s_{ir} \gets h(O_{ir})$
\ENDFOR
\ENDFOR
\STATE $\text{MixDiff} \gets \frac{1}{RN} \sum_{r=1}^R \sum_{i=1}^N (s_{ir} - s^*_{ir})$
\end{algorithmic}
\end{algorithm}

\paragraph{Practical implementation}
The oracle-side procedure can be precomputed since it does not depend on the target sample. The target-side computations can be effectively parallelized since each perturbed target sample can be processed by the model, independent of the others. We organize the perturbed target samples in a single batch in our implementation (see Appendix F for details on practical implementation). Further speedup can be gained in remote API environments as API calls are often handled by multiple nodes. 

\begin{figure*}[ht]
  \centering
  \begin{subfigure}[b]{0.23\linewidth}
    \centering
    \includegraphics[width=0.95\linewidth]{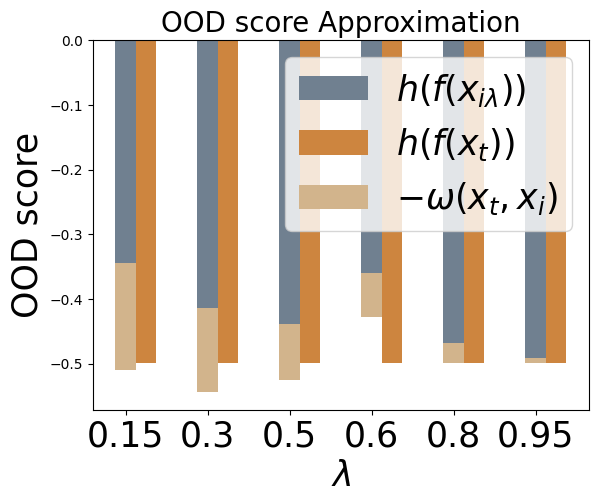} 
    \caption{} 
    \label{fig:theo_a} 
  \end{subfigure}
  \begin{subfigure}[b]{0.24\linewidth}
    \centering
    \includegraphics[width=0.98\linewidth]{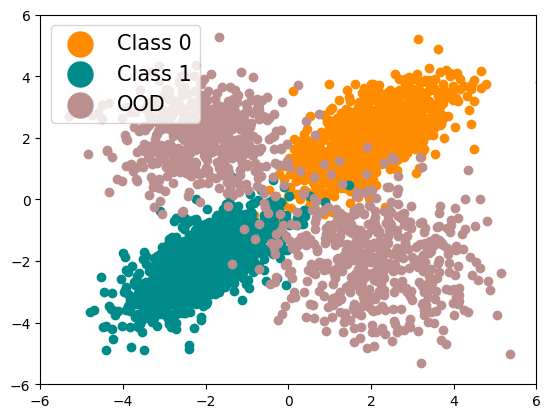} 
    \caption{} 
    \label{fig:theo_b} 
  \end{subfigure} 
  \begin{subfigure}[b]{0.24\linewidth}
    \centering
    \includegraphics[width=0.98\linewidth]{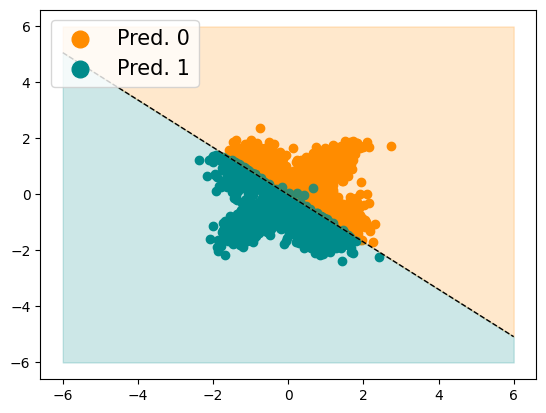}
    \caption{} 
    \label{fig:theo_c} 
  \end{subfigure}
  \begin{subfigure}[b]{0.24\linewidth}
    \centering
    \includegraphics[width=0.98\linewidth]{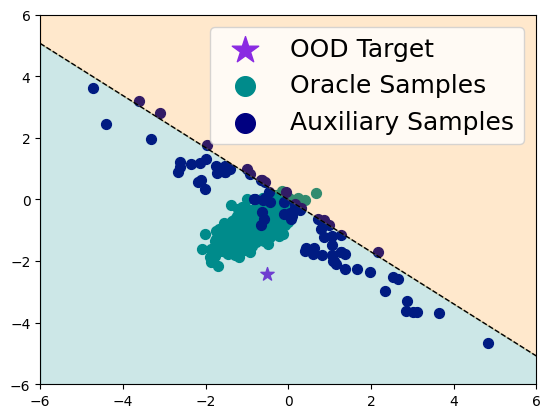} 
    \caption{} 
    \label{fig:theo_d} 
  \end{subfigure} 
  \vspace{10pt}
  \caption{\textbf{(a)} Approximation error for Equation \ref{eq:proposition_eq1} on synthetic data. Without higher-order terms, we can reasonably approximate the OOD score of mixed sample with decomposed terms. \textbf{(b)} The syntactic data distribution. Data is sampled from four independent Gaussian distributions, with two considered as ID samples for each class and the other two as OOD samples. We train a logistic regression model with this dataset. \textbf{(c)} The prediction results of the trained model. \textbf{(d)} Although the target sample is a hard OOD sample, there are auxiliary samples (blue dot) that guarantee that MixDiff is positive under some reasonable conditions introduced in Theorem \ref{theorem_mixdiff}.}
  \vspace{10pt}
  \label{fig:theo_plot} 
\end{figure*}

\subsection{Theoretical analysis}
To better understand how and when our method ensures performance improvements, we present a theoretical analysis of MixDiff.
We use a similar theoretical approach to \citet{zhang2021how}, but towards a distinct direction for analyzing a post hoc OOD scoring function. Proposition \ref{ood_func_mixup} reveals the decomposition of the OOD score function into two components: the OOD score of the unmixed clean target sample and the supplementary signals introduced by Mixup.

\begin{proposition} [OOD scores for mixed samples] \label{ood_func_mixup}
Let pre-trained model $f(\cdot)$ and base OOD score function $h(\cdot)$ be twice-differentiable functions, and $x_{i\lambda}=\lambda x_t + (1-\lambda)x_i$ be a mixed sample with ratio  $\lambda \in (0,1)$. Then OOD score function of mixed sample, $h(f(x_{i\lambda}))$, is written as:
{\small
\begin{equation} \label{eq:proposition_eq1}
    h(f(x_{i\lambda})) = h(f(x_t)) + \sum_{l=1}^3\omega_l(x_t, x_i) + \varphi_t(\lambda)(\lambda-1)^2,
\end{equation}
}
where $\lim_{\lambda\rightarrow 1}\varphi_t(\lambda)=0$,
\footnotesize{
\begin{align*}
&\omega_1(x_{t},x_{i}) = (\lambda-1)(x_t-x_i)^T{f}'(x_{t}){h}'(f(x_{t}))\\
&\omega_2(x_{t},x_{i}) = \frac{(\lambda-1)^2}{2}(x_t-x_i)^T{f}''(x_{t})(x_t-x_i){h}'(f(x_{t}))\\
&\omega_3(x_{t},x_{i}) =  \frac{(\lambda-1)^2}{2}(x_t-x_i)^T{f}'(x_{t})(x_t-x_i)^T{f}'(x_{t}){h}''(f(x_{t})).
\end{align*}
}
\end{proposition}

We analyze MixDiff using the quadratic approximation of $h(f(x_{i\lambda}))$, omitting the higher order terms denoted as $\varphi_t(\lambda)$ in Equation \ref{eq:proposition_eq1}. In Figure \ref{fig:theo_a}, we experimentally verify that the sum of the OOD score of the pure sample and $\omega$ terms, denoted as $\omega(x_t,x_i)=\sum_{l=1}^{3}\omega_l(x_t,x_i)$, reasonably approximates the OOD score of the mixed sample in Equation \ref{eq:proposition_eq1}. $\omega(x_t,x_i)$ represents the impact caused by Mixup as can be seen from its increase when $\lambda$ decreases. Hence, the additional signal from the Mixup can be derived from the first and second derivatives of $f(\cdot) \;\text{and}\; h(\cdot)$ and the difference between the target and auxiliary samples.

We argue that perturbing both the target and oracle samples and then comparing the model outputs of the two can help OOD detection even when the target induces a relatively high confidence score from the model, in which case existing output-based OOD scoring functions would result in detection failure. Through Theorem \ref{theorem_mixdiff}, we show the effectiveness of MixDiff by demonstrating the existence of an auxiliary sample with which MixDiff can calibrate the overconfidence of a high confidence OOD sample on a simple linear model setup.

\begin{theorem} \label{theorem_mixdiff}
Let h(x) represent MSP and f(x) represent a linear model, described by $w^Tx+b$, where $w, x \in \mathbb{R}^d$ and $b \in \mathbb{R}$. We consider the target sample, $x_t$, to be a hard OOD sample, defined as a sample that is predicted to be of the same class as the oracle sample, $x_m$, but with a higher confidence score than the oracle sample. For binary classification, $x_t$ is a hard OOD sample when $0 < f(x_m) < f(x_t) \text{ or } f(x_t) < f(x_m) < 0$. There exists an auxiliary sample $x_i$ such that
    \begin{equation*}\label{eq:mixdiff_ineq}
        h(f(x_t))-h(f(x_m))+\sum_{l=1}^3(\omega_l(x_t, x_i)-\omega_l(x_m, x_i)) > 0.
    \end{equation*}
\end{theorem}

Theorem \ref{theorem_mixdiff} provides a theoretical ground for our approach's effectiveness in discerning OOD samples that may not be detected by existing output-based OOD scores. \Cref{fig:theo_b,fig:theo_c,fig:theo_d} illustrate examples of such auxiliary samples using synthetic data. Proof and details of Proposition \ref{ood_func_mixup} and Theorem \ref{theorem_mixdiff} are in Appendix B and C, respectively. We also show that Theorem \ref{theorem_mixdiff} holds for MLS and Entropy in Appendix C. While we take a linear model as the classifier for simplicity of analysis, the prevalence of linear probing from foundation models' embeddings brings our analysis closer to real-world setups (see Section \ref{sub:emb_mixdiff} for experimental validation).

\section{Experiments} \label{sec:results}

\subsection{Experimental setup}
We elaborate on the implementation details and present the descriptions on baselines. Other details on datasets and evaluation metrics are provided in Appendix G. See Appendix O for code.

\paragraph{Implementation details} 
Following a recent OOD detection approach \citep{esmaeilpour2022zero, mingdelving, wang2023clipn} that utilizes vision-language foundation models' zero-shot classification capability, we employ CLIP ViT-B/32 model \citep{radford2021learning} as our classification model without any finetuning on ID samples. We construct the oracle set by randomly sampling $M$ samples per class from the train split of each dataset. For a given target sample, we simply use the other samples in the same batch as the auxiliary set.
Instead of searching hyperparameters for each dataset, we perform one hyperparameter search on Caltech101 \citep{fei2004learning} and use the same hyperparameters across all the other datasets, which is in line with a more realistic OOD detection setting
 \citep{liang2018enhancing}. 
 We provide full description of the implementation details in Appendix G.

\paragraph{Baselines} We take MSP \citep{hendrycks2017a}, MLS \citep{Hendrycks2022ScalingOD}, energy score \citep{liu2020energy}, Shannon entropy \citep{Thulasidasan2021AnEB} and MCM \citep{mingdelving} as output-based training-free baselines. We also include methods that require extra training for comparison. ZOC \citep{esmaeilpour2022zero} is a zero-shot OOD detection method based on CLIP \citep{radford2021learning} that requires training a separate candidate OOD class name generator. CAC \citep{miller2021class} relies on train-time loss function modification and shows the best performance among the train-time modification methods compatible with CLIP \citep{esmaeilpour2022zero}. We take CAC trained with the same CLIP ViT-B/32 backbone as a baseline (CLIP+CAC).

\begin{table*}[ht]
\caption{Average AUROC scores for five datasets. The highest and second highest AUROC scores from each block are highlighted with \textbf{bold} and \underline{underline}. The value on the right side of $\pm$ denotes the standard deviation induced from 5 different OOD, ID class splits. Statistically significant differences compared to the corresponding base score (indicated by background color) are \textit{italicised} (one-tailed paired $t$-test with $p$ $<$ 0.1). $\Delta$ represents difference from the corresponding base score. $\dagger$ indicates the reduced evaluation setting described in Appendix G.4. We report AUCPR, FPR95 scores in Appendix H. We report results on other CLIP backbones in Appendix M. Best viewed in color.} 
\centering
\resizebox{0.85\textwidth}{!}
{
\begin{tabular}{lcccccccc}
\toprule
  \textbf{Method} &
  \multicolumn{1}{l}{\textbf{Training-free}} &
  \multicolumn{1}{l}{\textbf{CIFAR10}} &
  \multicolumn{1}{l}{\textbf{CIFAR100}} &
  \multicolumn{1}{l}{\textbf{CIFAR+10}} &
  \multicolumn{1}{l}{\textbf{CIFAR+50}} &
  \multicolumn{1}{l}{\textbf{TinyImageNet}} &
  \multicolumn{1}{l}{\textbf{Avg.}} &
  \multicolumn{1}{c}{\textbf{$\Delta$}} \\ \midrule
CSI \citep{tack2020csi}                    & 
\ding{55} & 
87.0\scriptsize{$\pm$4.0}         & 
80.4\scriptsize{$\pm$1.0}           & 
94.0\scriptsize{$\pm$1.5}         & 
97.0\scriptsize                   & 
76.9\scriptsize{$\pm$1.2}           &
87.0       &
-       \\
CAC \citep{miller2021class}                & 
\ding{55} & 
80.1\scriptsize{$\pm$3.0}         & 
76.1\scriptsize{$\pm$0.7}           & 
87.7\scriptsize{$\pm$1.2}         & 
87.0\scriptsize                   & 
76.0\scriptsize{$\pm$1.5}           &
84.9       &
- \\

CLIP+CAC \citep{miller2021class}           & 
\ding{55} & 
89.3\scriptsize{$\pm$2.0}         & 
\textbf{83.5}\scriptsize{$\pm$1.2}  & 
96.5\scriptsize{$\pm$0.5}         & 
95.8\scriptsize         & 
\textbf{84.6}\scriptsize{$\pm$1.7}  & 
89.9       &
- \\

ZOC $^\dagger$ \citep{esmaeilpour2022zero} &
\ding{55} & 
{\underline{91.5}\scriptsize{$\pm$}2.5}   & 
82.7\scriptsize{$\pm$2.8}           & 
{\underline {97.6}\scriptsize{$\pm$1.1}}   & 
\underline {97.1}\scriptsize   &  
82.6\scriptsize{$\pm$3.1}     & 
\underline {90.3} &
-
\\ \midrule

MixDiff+ZOC $^\dagger$ &
  \ding{55}                           &
  \textbf{92.2}\scriptsize{$\pm$2.5}  &
  {\underline {82.8}\scriptsize{$\pm$2.4}}     & 
  \textbf{98.2}\scriptsize{$\pm$1.2}  &
  \textbf{98.5}\scriptsize            &
  {\underline {82.9}\scriptsize{$\pm$3.3}}     &
  \textbf{90.9} & 
  +0.6 \\  \midrule \midrule 
  
\rowcolor{MSP} 
MSP \citep{hendrycks2017a}                  
    & \ding{51} 
    & 88.7\scriptsize{$\pm$2.0}       
    & 78.2\scriptsize{$\pm$3.1}        
    & 95.0\scriptsize{$\pm$0.8}        
    & 95.1\scriptsize                  
    & 80.4\scriptsize{$\pm$2.5}          
    & 87.5       
    & - \\
    
\rowcolor{MLS} 
MLS \citep{Hendrycks2022ScalingOD}          
& \ding{51} 
& 87.8\scriptsize{$\pm$3.0}       
& 80.0\scriptsize{$\pm$3.1}        
& 96.1\scriptsize{$\pm$0.8}        
& 96.0\scriptsize                  
& \underline{84.0}\scriptsize{$\pm$1.2}    
& 88.8      
& - 
\\

\rowcolor{Energy} 
Energy \citep{liu2020energy}                
& \ding{51} 
& 85.4\scriptsize{$\pm$3.0}       
& 77.6\scriptsize{$\pm$3.7}        
& 94.9\scriptsize{$\pm$0.9}        
& 94.8\scriptsize                  
& 83.2\scriptsize{$\pm$1.2}          
& 87.2      
& -
\\

\rowcolor{Entropy} 
Entropy \citep{Thulasidasan2021AnEB}  & 
    \ding{51} & 
    89.9\scriptsize{$\pm$2.6} & 
    79.9\scriptsize{$\pm$2.5}        & 
    96.8\scriptsize{$\pm$0.8}  & 
    96.8\scriptsize        & 
    82.2\scriptsize{$\pm$2.3}          & 
    89.1       & 
    - \\
    
\rowcolor{MCM} 
MCM \citep{mingdelving}  &
    \ding{51} & 
    90.6\scriptsize{$\pm$2.9} & 
    80.3\scriptsize{$\pm$2.1}        & 
    96.9\underline{}\scriptsize{$\pm$0.8}  & 
    97.0\scriptsize        & 
    83.1\scriptsize{$\pm$2.2}          & 
    89.6 &
    - 
    \\ \midrule
    
\rowcolor{MSP} 
MixDiff+MSP                                & 
    \ding{51} & 
    \textit{89.2\scriptsize{$\pm$1.6}}  & 
    \textit{80.1\scriptsize{$\pm$2.8}} & 
    \textit{96.7\scriptsize{$\pm$0.8}} & 
    96.9\scriptsize & 
    \textit{81.6\scriptsize{$\pm$2.6}} & 
    88.9       &
    +1.4 \\
    
\rowcolor{MLS} 
MixDiff+MLS                                & 
    \ding{51} & 
    87.9\scriptsize{$\pm$2.1}       & 
    80.5\scriptsize{$\pm$2.2}  & 
    \textit{96.5\scriptsize{$\pm$0.7}} & 
    96.9\scriptsize  & 
    \textbf{84.5}\scriptsize{$\pm$}0.9 & 
    89.3 &
    +0.5 \\
    
\rowcolor{Energy} 
MixDiff+Energy                             & 
    \ding{51} & 
    85.6\scriptsize{$\pm$2.2}       & 
    78.3\scriptsize{$\pm$2.7}        & 
    \textit{95.4\scriptsize{$\pm$0.8}} & 
    95.9\scriptsize        & 
    83.6\scriptsize{$\pm$1.1}          & 
    87.8       &
    +0.6 
    \\
\rowcolor{Entropy} 
MixDiff+Entropy &
  \ding{51}      &
  \textit{\underline{90.7}\scriptsize{$\pm$1.8}}  &
  \underline{81.0}\scriptsize{$\pm$2.6}     &
  \textit{\textbf{97.6}\scriptsize{$\pm$0.8}} &
  \underline{97.6}\scriptsize  &
  82.9\scriptsize{$\pm$2.4}           &
  \underline{90.0}  &
  +0.9 \\
  
\rowcolor{MCM} 
MixDiff+MCM  &
    \ding{51} & 
    \textbf{91.4}\scriptsize{$\pm1.8$} & 
    \textbf{81.4}\scriptsize{$\pm2.6$}        & 
    \textit{\underline{97.5}\scriptsize{$\pm0.9$}} & 
    \textbf{97.7}\scriptsize        & 
    83.9\scriptsize{$\pm2.2$}          & 
    \textbf{90.4} &
    +0.8 
    \\ \bottomrule
\end{tabular}
}
\label{tab:main_five_auroc}
\end{table*}

\subsection{Logits as model outputs}
First, we assume a more lenient access constraint whereby logits are provided as the model $f(\cdot)$'s outputs. This setup facilitates validation of MixDiff's OOD score enhancement ability on both the logit-based and probability-based scores. Note that, in this setup, the perturbed oracle samples' probability-based OOD scores are computed after averaging out $M$ perturbed oracle samples in the logit-space, \textit{i.e.}, $\Bar{O}^{*}_{ir} = \frac{1}{M}\sum_{m=1}^{M}O^{*}_{mir}$. 
The consistent improvements across all datasets and methods in Table \ref{tab:main_five_auroc} indicate that MixDiff is effective in enhancing output-based OOD scores, to a degree where one of the training-free methods, MixDiff+MCM, outperforming a training-based method CLIP+CAC.
Equipping MixDiff with the best performing non-training-free method, ZOC, also yields performance improvements. 

\begin{table*}[ht]
\caption{AUROC scores on various degrees of model access scenarios. The methods in the bottom block require the model's inner activations and are evaluated with the same CLIP ViT-B/32 backbone and entropy as OOD scoring function. Best viewed in color.}
\centering
\resizebox{\textwidth}{!}
{
\begin{tabular}{lcccccccc}
\toprule
  \textbf{Method} &
  \multicolumn{1}{c}{\textbf{Access}} &
  \multicolumn{1}{l}{\textbf{CIFAR10}} &
  \multicolumn{1}{l}{\textbf{CIFAR100}} &
  \multicolumn{1}{l}{\textbf{CIFAR+10}} &
  \multicolumn{1}{l}{\textbf{CIFAR+50}} &
  \multicolumn{1}{l}{\textbf{TinyImageNet}} &
  \multicolumn{1}{l}{\textbf{Avg.}} &
  \multicolumn{1}{c}{\textbf{$\Delta$}} \\ \midrule
  
\rowcolor{MSP} 
MSP \citep{hendrycks2017a}                  & 
Prediction prob. & 88.7\scriptsize{$\pm$2.0}       & 
78.2\scriptsize{$\pm$3.1}        & 
95.0\scriptsize{$\pm$0.8}        & 
95.1\scriptsize        & 
80.4\scriptsize{$\pm$2.5}          & 
87.5 &
- \\
\rowcolor{Entropy} 
Entropy \citep{Thulasidasan2021AnEB}  & 
Prediction prob. & 89.9\scriptsize{$\pm$2.6} & 
79.9\scriptsize{$\pm$2.5}        & 
96.8\scriptsize{$\pm$0.8}  & 
96.8\scriptsize        & 
82.2\scriptsize{$\pm$2.3}          & 
89.1 &
- \\
Entropy+MSP & 
Prediction prob. &
89.9\scriptsize{$\pm$2.6} & 
79.9\scriptsize{$\pm$2.5} & 
96.8\scriptsize{$\pm$0.8} & 
96.8\scriptsize & 
82.2\scriptsize{$\pm$2.3} & 
89.1 & 
+0.0 \\
\midrule
\rowcolor{MSP} 
MixDiff+MSP (Prediction prob.)      & 
Prediction prob. &
\textit{89.4\scriptsize{$\pm$1.3}} & 
\textit{80.0\scriptsize{$\pm$2.8}} & 
\textit{96.5\scriptsize{$\pm$0.8}} & 
96.8\scriptsize & 
\textit{\textit{81.8\scriptsize{$\pm$2.4}}} & 
88.9 &
+1.4 \\
\rowcolor{Entropy} 
MixDiff+Entropy (Prediction prob.) &
Prediction prob. &
\textit{\textbf{91.1}\scriptsize{$\pm$1.6}} & 
\underline{80.9}\scriptsize{$\pm$2.6} & 
\textit{\underline{97.1}\scriptsize{$\pm$0.8}} & 
\underline{97.3}\scriptsize & 
\underline{82.9}\scriptsize{$\pm$2.3} & 
\textbf{89.9} &
+0.8 \\
\ \ with oracle as auxiliary &
Prediction prob. &
90.6\scriptsize{$\pm$1.7} & 
\textbf{81.1}\scriptsize{$\pm$2.0} & 
\textbf{97.3}\scriptsize{$\pm$0.7} & 
\textbf{97.4}\scriptsize & 
82.9\scriptsize{$\pm$2.2} & 
\textbf{89.9} & 
+0.8 \\
\ \ with random ID as auxiliary &
Prediction prob. &
90.8\scriptsize{$\pm$1.5} & 
\textbf{81.1}\scriptsize{$\pm$2.1} & 
96.8\scriptsize{$\pm$1.0} & 
96.8\scriptsize & 
\underline{82.9}\scriptsize{$\pm$2.3} & 
\underline{89.7} &
+0.6 \\
\ \ with unlabeled oracle &
Prediction prob. &
\underline{91.0}\scriptsize{$\pm$1.6} & 
80.5\scriptsize{$\pm$2.9} & 
\underline{97.1}\scriptsize{$\pm$0.8} & 
\underline{97.3}\scriptsize & 
82.7\scriptsize{$\pm$2.1} & 
\underline{89.7} & 
+0.6 \\
\ \ without compare part &
Prediction prob. &
89.4\scriptsize{$\pm$2.9} & 
79.5\scriptsize{$\pm$2.7} & 
\underline{97.1}\scriptsize{$\pm$0.9} & 
97.2\scriptsize & 
81.6\scriptsize{$\pm$2.5} & 
89.0 &
-0.1 \\
\ \ without oracle selection &
Prediction prob. &
89.5\scriptsize{$\pm$2.8} & 
79.6\scriptsize{$\pm$2.7} & 
\underline{97.1}\scriptsize{$\pm$0.9} & 
\underline{97.3}\scriptsize & 
81.7\scriptsize{$\pm$2.5} & 
89.0 &
-0.1 \\

\midrule \midrule

Random score from uniform dist. &
Prediction label  &
49.6\scriptsize{$\pm$0.5} & 
49.8\scriptsize{$\pm$1.1} & 
49.8\scriptsize{$\pm$0.7} & 
50.1\scriptsize & 
49.8\scriptsize{$\pm$0.4} & 
49.8 & 
- \\ \midrule

MixDiff with random ID as auxiliary &
Prediction label &
\textbf{62.4}\scriptsize{$\pm$4.1} & 
\textbf{59.4}\scriptsize{$\pm$6.2} & 
\textbf{65.6}\scriptsize{$\pm$1.5} & 
\textbf{65.4}\scriptsize & 
\textbf{63.3}\scriptsize{$\pm$2.8} & 
\textbf{63.2} &
+13.4 \\
MixDiff with oracle as auxiliary &
Prediction label &
61.9\scriptsize{$\pm$3.7} & 
55.1\scriptsize{$\pm$7.1} & 
59.9\scriptsize{$\pm$1.1} & 
59.8\scriptsize & 
55.6\scriptsize{$\pm$2.7} & 
58.4 & 
+8.6 \\

\midrule \midrule

\rowcolor{MSP} 
MixDiff+MSP (Embedding Mixup) &
Activation &
\textit{90.0}\scriptsize{$\pm1.8$} & 
80.0\scriptsize{$\pm3.6$} & 
\textit{95.6}\scriptsize{$\pm0.8$} & 
95.7\scriptsize & 
\textit{82.2}\scriptsize{$\pm2.3$} &
88.7&
+1.2 \\

\rowcolor{Entropy} 
MixDiff+Entropy (Embedding Mixup) &
Activation &
\textit{\textbf{91.1}}\scriptsize{$\pm2.0$} & 
\textbf{81.1}\scriptsize{$\pm3.2$} & 
\textit{\textbf{97.1}}\scriptsize{$\pm0.7$} & 
\textbf{97.1}\scriptsize & 
\textit{83.7}\scriptsize{$\pm2.2$} &
\textbf{90.0} &
+0.9 \\

DML \citep{Zhang_2023_CVPR} &
Activation &
87.8\scriptsize{$\pm3.0$} & 
80.0\scriptsize{$\pm3.1$} & 
96.1\scriptsize{$\pm0.8$} & 
96.0\scriptsize & 
\textbf{84.0}\scriptsize{$\pm1.2$} &
88.8 &
- \\

ASH \citep{djurisic2023extremely} &
Activation &
85.2\scriptsize{$\pm3.8$} & 
75.4\scriptsize{$\pm4.4$} & 
92.5\scriptsize{$\pm0.9$} & 
92.4\scriptsize & 
77.2\scriptsize{$\pm3.1$} &
84.5 &
- \\

 \bottomrule
\end{tabular} 
}
\label{tab:softmax_results}
\end{table*}

\subsection{Prediction probabilities as model outputs}
We now take a more restricted environment where the only accessible part of the model is its output prediction probabilities. To the best of our knowledge, none of the existing OOD score enhancement methods are applicable in this environment. Logits are required in the case of Softmax temperature scaling \citep{liang2018enhancing}. ODIN's gradient-based input preprocessing \citep{liang2018enhancing} or weight pruning methods \citep{sun2022dice} assume an access to the model's parameters. The model's internal activations are required in the case of activation clipping \citep{sun2021react} and activation pruning \citep{djurisic2023extremely}.

\begin{figure}[ht]
\centering

\begin{subfigure}[b]{0.49\columnwidth}
\centering
\resizebox{1.0\textwidth}{!}{%
\begin{tikzpicture}
    \begin{axis}[
        view={-20}{20}, grid=both,
        xlabel = $N$,
        xlabel style = {font=\huge},
        ylabel = $R$,
        ylabel style = {font=\huge},
        zlabel = AUROC,
        zlabel style = {font=\huge},
        xticklabel style = {font=\huge},
        yticklabel style = {font=\huge},
        zticklabel style = {font=\huge},
        xtick = {1,2,5,8,11,14},
        ytick ={1,3,5,7},
        zmin = 79.5, zmax = 81.5
    ]
      
      \addplot3 [
          surf,
          samples=5,
          domain=1:14,
          domain y=1:7,
          fill opacity = 0.3,
      ]
      (x,y,79.86);

      \addplot3[
      surf,
      ]
      file {filename.txt};
    \end{axis}
\end{tikzpicture}
}
\caption{}
\label{fig:auroc_by_rn}
\end{subfigure}
\begin{subfigure}[b]{0.49\columnwidth}
\centering
\resizebox{1.0\textwidth}{!}{%
\begin{tikzpicture}
  \centering
  \begin{axis}[
        ybar, axis on top,
        bar width=0.6cm,
        ymajorgrids, tick align=inside,
        major grid style={draw=white},
        enlarge y limits={value=.1,upper},
        ymin=0, ymax=7,
        axis x line*=bottom,
        axis y line*=right,
        y axis line style={opacity=0},
        tickwidth=0pt,
        tick label style={font=\large},
        enlarge x limits=true,
        legend style={
            at={(0,0.95)},
            anchor=north west,
            legend columns=1,
            style={nodes={scale=1.5, transform shape}},
            /tikz/every even column/.append style={column sep=0.5cm}
        },
        ylabel={OOD, ID Difference},
        ylabel style = {font=\huge},
        xlabel={Avg. MSP Value in Interval},
        xlabel style = {font=\huge},
        symbolic x coords={
          -99,
          -98,
          -97,
          -96,
          -94,
          },
       xtick=data,
       nodes near coords={
        \pgfmathprintnumber[precision=1]{\pgfplotspointmeta}
       }
    ]
    \addplot [draw=none, fill=mixdiff_bar] coordinates {
      (-99, 00.71282)
      (-98, 02.132501) 
      (-97, 02.576212)
      (-96, 07.040526) 
      (-94, 03.051829) };
   \addplot [draw=none,fill=blue!100] coordinates {
      (-99, 00.009)
      (-98, 00.003) 
      (-97, 00.139)
      (-96, 00.050) 
      (-94, 00.066) };
    \legend{MixDiff, MSP}
  \end{axis}
  \end{tikzpicture}

}
\caption{}
\label{fig:ood_id_diff}
\end{subfigure}
\subfigspace
\caption{Additional analyses on CIFAR100. \textbf{(a)} AUROC scores of MixDiff+Entropy with varying values of $N$ and $R$ (top). AUROC score of Entropy (bottom). We also provide processing time analysis in Appendix K. \textbf{(b)} Difference of the OOD, ID samples' average uncertainty scores belonging to a given interval of MSP score. None-overlapping five consecutive intervals whose values lie below the threshold set by FPR95 are constructed. MixDiff scores can  discriminate OOD, ID samples even when its base score values are almost identical.} 
\figspace
\vspace{5px}
\end{figure}
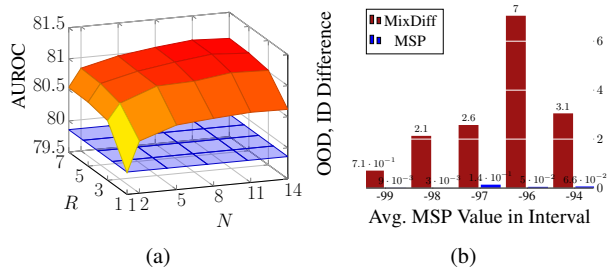

We take a linear combination of entropy and MSP scores with a scaling hyperparameter tuned on the Caltech101 dataset as a baseline (Entropy+MSP). The results are presented in Table \ref{tab:softmax_results}. Even in this constrained environment, MixDiff effectively enhances output-based OOD scores, as evidenced by MixDiff+Entropy outperforming MCM (in Table \ref{tab:main_five_auroc}), a method that assumes an access to the logit space, while MSP score fails to provide entropy score any meaningful performance gain. Figure \ref{fig:auroc_by_rn} shows that MixDiff's  performance gain can be enjoyed with as little as two additional forward passes ($R=1$, $N=2$). Figure \ref{fig:ood_id_diff} illustrates the discriminative edge provided by MixDiff score when the base OOD score's values are almost identical. We observe that the performance gain is more pronounced when the outputs contain more limited information as can be seen in the case of MSP where only the predicted class's probability value is utilized. 

\paragraph{Ablations}
We present the ablation results of MixDiff framework in Table \ref{tab:softmax_results} to illuminate each component's effect on performance. We take MixDiff+Entropy for these experiments. MixDiff's performance improvements are consistent when the homogeneity of auxiliary samples is gradually increased by changing the in-batch auxiliary samples, which may contain OOD samples, to random ID samples (\textit{random ID as auxiliary}), and to the other oracle samples with the same predicted label as the target (\textit{oracle as auxiliary}), suggesting that MixDiff is robust to the choice of auxiliary samples. Eliminating the comparison part in the perturb-and-compare approach by using only the perturbed target's scores without comparing with the perturbed oracles' scores (\textit{without compare part}), and randomly choosing oracle samples from a set of ID sample instead of finding similar oracle samples using the predicted class label (\textit{without oracle selection}) result in performance degradation. These observations suggest that comparing the relative change from a similar ID sample is crucial. We show that MixDiff is applicable even when there is no labeled oracle set by selecting top-$M$ most similar samples from $M \times K$ unlabeled ID samples with similarity calculated from the dot product of the prediction probabilities of the target and the unlabeled oracle samples (\textit{unlabeled oracle}).

\subsection{Prediction labels as model outputs}
We push the limits of the model access by assuming that only the predicted class labels are available without any scores attached to them. We apply MixDiff by representing the model's predictions as one-hot vectors and taking the difference between the perturbed target's predicted label and the corresponding perturbed oracles' average score for that label in Equation \ref{eq:mixdiff}. As there is no base OOD score applicable in the environment, we use the MixDiff score alone. The results in Table \ref{tab:softmax_results} show that MixDiff is applicable even in this extremely constrained access environment.

\subsection{Last layer activations as model outputs} \label{sub:emb_mixdiff}
We relax the model access constraint by permiting access to the model's activations from the last layer, \textit{i.e.}, image embeddings in CLIP model. In this setup, instead of input-level Mixup, we utilize embedding-level Mixup. More specifically, embeddings of target (or oracle) are perturbed by mixing them with auxiliary sample's embeddings, after which logits are computed from the perturbed embeddings and fed to an output-based OOD scoring function $h(\cdot)$ such as entropy. As auxiliaries' and oracles' embeddings are precomputed, the computational overhead introduced by MixDiff is almost nil. The assumption of linear model in theoretical analysis is more closely followed in embedding-level Mixup since they can be viewed as linear probing of foundation models' activations. Bottom block of Table \ref{tab:softmax_results} shows that MixDiff can enhance OOD detection performance even with negligible compute overhead in this relaxed setup. We use random ID samples as auxiliaries in the embedding Mixup experiments.

\subsection{Robustness to adversarial attacks}
In adversarial attack on an OOD detector, the attacker creates a small, indistinguishable modification to the input sample with the purpose of increasing the model's confidence of a given OOD sample or decreasing the model's confidence of a given ID sample \citep{azizmalayeri2022your}.
These modifications can be viewed as injection of certain artificial features, specifically designed to induce more confident or uncertain outputs from the model. Our motivation in Section \mbox{\ref{sec:intro}} suggests that these artificial features may also be less robust to perturbations. We test this by evaluating MixDiff under adversarial attack. The results in Table \mbox{\ref{tab:adv}} indicate that the contributing features that induce ID/OOD misclassification are less robust to perturbations and that MixDiff can effectively exploit such brittleness.
Detailed description of the experimental setup is in Appendix G.5.

\begin{table}[h]
  \caption{AUROC scores on various attack scenarios. "In" (or "Out") indicates all of the ID (or OOD) samples are adversarially modified. "Both" indicates all of the ID, OOD samples are adversarially modified. "MixDiff Only" refers to the score in Equation \ref{eq:mixdiff} with entropy as the OOD scoring function $h(\cdot)$.}
  \label{tab:adv}%
\resizebox{\columnwidth}{!}
{
    \begin{tabular}{l|cccc|cccc}
    \toprule 
\multicolumn{1}{l}{\multirow{2}{*}{\textbf{Method}}}  & \multicolumn{4}{c|}{\textbf{CIFAR10}} & \multicolumn{4}{c}{\textbf{CIFAR100}} \\
\cmidrule{2-9}    \multicolumn{1}{l}{} & \textbf{Clean} & \textbf{In} & \textbf{Out} & \textbf{Both} & \textbf{Clean} & \textbf{In}    & \textbf{Out}   & \textbf{Both} \\
    \midrule
    Entropy & \underline{89.88} & 47.42  & 13.77  & 2.68  & \underline{79.87} & 36.86 & 14.38 & 2.21 \\
    MixDiff+Entropy & \textbf{90.64} & \underline{54.71} & \underline{31.77} & \underline{8.84}  & \textbf{81.11} & \underline{50.31} & \underline{31.40}  & \underline{9.08}  \\
    MixDiff Only & 88.16 & \textbf{61.00} & \textbf{40.28} & \textbf{20.45} & 78.05 & \textbf{58.84}  & \textbf{44.19} & \textbf{27.48}   \\
    \bottomrule
    \end{tabular}%
}
\end{table}%

\subsection{Experiments on out-of-scope detection task}

\paragraph{Out-of-scope detection}
We take the MixDiff framework to out-of-scope (OOS) detection task to check its versatility in regard to the modality of the input. To reliably fulfill users' queries or instructions, understanding the intent behind a user's utterance forms a crucial aspect of dialogue systems. In intent classification task, models are tasked to extract the intent behind a user utterance. Even though there has been an inflow of development in the area for the improvement of classification performance, 
there is no guarantee that a given query's intent is in the set of intents that the model is able to classify.
OOS detection task \citep{chen-yu-2021-gold}, 
concerns with detection of such user utterances. 

\paragraph{MixDiff with textual input}
Unlike images whose continuousness lends itself to a simple Mixup operation, the discreteness of texts renders Mixup of texts not as straightforward. While there are several works that explore interpolation of texts, most of these require access to the model parameters \citep{kong-etal-2022-dropmix}.
This limits the MixDiff framework's applicability in an environment where the model is served as an API \citep{sun2022bbt}, 
which is becoming more and more prevalent with the rapid development of large language models \citep{openai2023gpt4}.
Following this trend, we assume a more challenging environment with the requirement that Mixup be performed on the input level. To this end, we simply concatenate the text pair and let the interpolation happen while the pair is inside the model \citep{Hao_2023_WACV}. 

\begin{table}[h]
\caption{Average AUROC scores for out-of-scope detection task. }
\centering
{
\resizebox{\columnwidth}{!}
{
\begin{tabular}{lccccc}
\toprule
\textbf{Method}                                                                   & \textbf{CLINC150} 
                                                                                                           & \textbf{Banking77}                 & \textbf{ACID}                        & \textbf{TOP}                            & \textbf{Average}        \\ \midrule
                                                                                                           
\multirow{1}{*}{MSP}                  
                                                                         & 93.02                      & 85.43                      &  88.98                       & 90.01                       & 89.36  \\
                                                            
\multirow{1}{*}{MLS}          
                                                                         & 93.56                      & 85.02                      &  88.91                       & 90.06                       & 89.39  \\ 
                          
\multirow{1}{*}{Energy}               
                                                                         & 93.61                      & 84.99                      &  88.83                       & 90.06                       & 89.37  \\ 
                                                            
\multirow{1}{*}{Entropy}  
                                                                         & 93.29                      & 85.59                      &  88.87                       & 90.02                       & 89.44  \\ \midrule
                                                            
\multirow{1}{*}{MixDiff+MSP}                                
                                                                         & 93.42                      & \underline{85.75}          &  \underline{89.18}           & \textbf{90.68}              & \underline{89.76}  \\ 
                    
\multirow{1}{*}{MixDiff+MLS}                                
                                                                         & \underline{93.88}          & 85.46                      &  \textbf{89.24}              & \underline{90.35}           & 89.73  \\ 
                          
\multirow{1}{*}{MixDiff+Energy}                             
                                                                         & \textbf{93.89}             & 85.51                      &  \underline{89.18}           & \underline{90.35}           & 89.73  \\ 
                      
\multirow{1}{*}{MixDiff+Entropy}                            
                                                                         & 93.67                      & \textbf{85.98}             &  89.13                       & \textbf{90.68}              & \textbf{89.87}  \\ \bottomrule
                        
\end{tabular}%
}
}

\label{tab:text_auroc}
\end{table}

\paragraph{Experimental setup}
We run OOS detection experiments using 4 intent classification datasets: CLINC150, Banking77, ACID, TOP. 
Following \citet{zhan-etal-2021-scope}, 
we randomly split the provided classes into in-scope and OOS intents, with in-scope intent class ratios of 25\%, 50\%, 75\%. For the intent classification model, we finetune BERT-base model \citep{devlin-etal-2019-bert} on the in-scope split of each dataset's train set. For each in-scope ratio, we construct 10 in-scope, OOS splits with different random seeds. Detailed experimental setup is in Appendix G.6.

\paragraph{Results}
We report the average AUROC scores in Table \ref{tab:text_auroc}, each of which is averaged over the in-scope class ratios as well as the class splits. Even with a simple Mixup method that simply concatenates the text pair, MixDiff consistently improves the performance across diverse datasets. The results suggest that the MixDiff framework's applicability is not limited to images and that the framework can be applied to other modalities with an appropriate perturbation method.

\section{Liminations and future work}

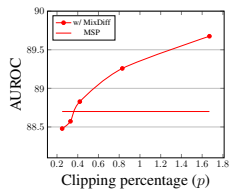
\begin{wrapfigure}{r}{3cm}
\vspace{-4mm}
\centering
\resizebox{0.37\columnwidth}{!}{%
    \begin{tikzpicture}
    \begin{axis}[
        ymin=88.1, ymax=90,
        ylabel=AUROC,
        ymajorgrids=true,
        xlabel={Clipping percentage ($p$)},
        legend pos=north west,
        label style={font=\huge},
        tick label style={font=\large},
    ]
    \addplot[smooth,mark=*,red] 
      coordinates{
        (0.25,88.479)
        (0.33,88.573) 
        (0.42,88.828)
        (0.83,89.258)
        (1.67,89.675) 
    }; \addlegendentry{w/ MixDiff}
    
      \addplot[mark=none, red] coordinates {(0.25,88.7) (1.67,88.7)};
      \addlegendentry{MSP}
    
    \end{axis}
    \end{tikzpicture}
}
\caption{Effect of low-confidence oracles.}
\label{fig:low_conf}
\vspace{-10px}
\end{wrapfigure}

\paragraph{Dependency on model's performance} We construct a low-confidence oracle set by limiting the oracle pool to contain the top $p\%$ of most uncertain ID samples. Fig. \ref{fig:low_conf} shows MixDiff's dependency on the model's ability to assign minimal confidence on the oracle. The experiments are performed with CIFAR10 dataset using the other oracle samples of the predicted class of the target as auxiliaries.

\paragraph{Time and space complexity} MixDiff is effective at bypassing a black-box model's access restriction for OOD detection, but bypassing the access restriction comes with a certain computational overhead. For each target sample $x_t$, MixDiff requires processing of $N \times R$ mixed samples. While these samples can be effectively processed in parallel and the MixDiff framework outperforming the baselines only with small values of $R$ and $N$, it nonetheless remains as a drawback of the MixDiff framework. Further research is called for reducing the computational and space complexity of MixDiff framework. 

\paragraph{Selection of auxiliary samples}
In Section \ref{sec:results}, we experiment with three auxiliary sample selection methods, one using the in-batch samples and the other two using the oracle or random ID samples as the auxiliary samples. Figure \ref{fig:auroc_by_rn} shows reduced performance gain when the number of auxiliary samples, $N$, is too small. We hypothesize that this is due to the fact that while on average MixDiff can effectively discern the overemphasized features, there is a certain degree of variance in the MixDiff score, requiring $N$ and, to some degree, $R$ to be over a certain value for reliable performance. There may be an auxiliary sample that is more effective at discerning an overemphasized feature of a given target sample, but this is subject to change depending on the target sample. We leave the exploration of better auxiliary sample selection methods, either by careful curation of auxiliary samples or by making the procedure more instance-aware and possibly learnable, as future work.

\paragraph{Other forms of inputs}
MixDiff framework can be easily extended to incorporate inputs from other modalities. The experiments on the out-of-scope detection task serve as an example of these kinds of extensions. This input-level Mixup makes the framework applicable to environments where the access to the model parameters cannot be assumed. It also grants the freedom to design better Mixup methods that are specific to the format of the input or the task at hand. But this freedom comes at the cost of having to devise a Mixup mechanism for each input format and task. For example, the simple concatenation of samples that we have utilized on out-of-scope detection task has the limitation that it cannot be applied if the input sequence is too long due to the quadratic time and space complexity of Transformers \citep{NIPS2017_3f5ee243}. 

\paragraph{Other types of distribution shifts and broader categories of models}
This work deals with detecting label shift with classifier models. However, there are other types of distribution shifts such as domain shift and broader range of models other than classifiers, \textit{e.g.}, image segmentation models. Extensions of the perturb-and-compare mechanism to more diverse types of shifts and tasks would be a valuable addition to the black-box OOD detection field.

\section{Conclusion}
In this work, we present a new OOD detection framework, MixDiff, that boosts OOD detection performance in constrained access scenarios. MixDiff is based on the perturb-and-compare approach that measures how the model's confidence in the target sample behaves compared to a similar ID sample when both undergo an identical perturbation. This provides an additional signal that cannot be gained from the limited information of the target sample's model output alone. We provide theoretical grounds for the framework's effectiveness and empirically validate our approach on multiple degrees of restricted access scenarios. Our experimental results show that MixDiff is an effective OOD detection method for constrained access scenarios where the applicability of existing methods is limited. 

\begin{ack}
This work was supported by the National Research Foundation of Korea (NRF) grant funded by the Korea government (MSIT) (No. 2021R1F1A1060117, 2022R1A4A3033874) and ICT Creative Consilience Program through the Institute of Information \& Communications Technology Planning \& Evaluation (IITP) grant funded by the Korea government (MSIT) (IITP-2024-2020-0-01821).
\end{ack}




\bibliography{m1031}

\end{document}


\date{}
\maketitle

\appendix

\tableofcontents
\addtocontents{toc}{\vskip-\ht\strutbox\makebox[\linewidth][c]{\rule{\dimexpr\linewidth}{1pt}}\par}
\vskip5pt\hrule\vskip5pt

\section{Notation}

\begin{table}[H]
\centering
\label{tab:notation}
\resizebox{\textwidth}{!}{%
\begin{tabular}{ll}
\toprule
Notation             & Definition                                                                                                             \\ \midrule
$f(\cdot)$           & Classifier model.                                                                                                          \\
$h(\cdot)$           & Arbitrary output-based OOD score function.                                                                                           \\
$M$                  & The number of oracle samples of each class.                                                                            \\
$R$                  & The number of Mixup ratios.                                                                                              \\
$N$                  & The number of auxiliary samples that will be mixed with the oracle or target samples.                                  \\
$\Omega_k$           & Set of oracle sample and label pairs for the $k$-th class.                                                                 \\
$\Omega$             & Set of oracle sample and label pairs of all classes, $\{ \Omega_k \}_{k=1}^K$. \\
$\lambda_r$          & $r$-th Mixup ratio.                                                                                 \\
$x_t$                & The target sample. \\
$x_{ir}$             & Mixed sample from the target $x_t$ and $i$-th auxiliary sample with Mixup ratio of $\lambda_r$. \\
$x^{*}_{mir}$        & Mixed sample from the $m$-th oracle sample $i$-th auxiliary sample with Mixup ratio of $\lambda_r$. \\
$O_{ir}$             & The prediction scores from the mixture of the target and $i$-th auxiliary sample with the Mixup ratio $\lambda_r$.                  \\
$O^{*}_{mir}$        & The prediction scores from the mixture of the $m$-th oracle and $i$-th auxiliary sample with the Mixup ratio $\lambda_r$.     \\
$\bar{O}^{*}_{ir}$   & The mean of \{$O^{*}_{mir}\}_{m=1}^{M}$ along the subscript $m$.                                             \\
$s_{ir}$             & OOD score induced by $O_{ir}$.                                                                                    \\
$s^{*}_{ir}$         & OOD score induced by $\bar{O}^{*}_{ir}$.                                                                           \\
$\gamma$             & The scaling hyperparameter to which the MixDiff score will be multiplied.                                               \\
\bottomrule
\end{tabular}%
}
\end{table}

\section{Proof of Proposition \ref{ood_func_mixup_app}} \label{app:prop_proof}
    
\begin{proposition} [OOD score function for mixed samples] \label{ood_func_mixup_app}

Let pre-trained model $f(\cdot)$ and base OOD score function $h(\cdot)$ be a twice-differentiable function, and $x_{i\lambda}=\lambda x_t + (1-\lambda)x_i$ be a mixed sample with ratio  $\lambda \in (0,1)$. Then base OOD score function of mixed sample, $h(f(x_{i\lambda}))$, is written as:

\begin{equation} \label{eq:proposition_eq1_app}
    h(f(x_{i\lambda})) = h(f(x_t)) + \sum_{l=1}^3\omega_l(x_t, x_i) + \varphi_t(\lambda)(\lambda-1)^2
\end{equation}
where $\lim_{\lambda\rightarrow 1}\varphi_t(\lambda)=0$,
\begin{align*}
\omega_1(x_t, x_i) &= (\lambda-1)(x_t-x_i)^T{f}'(x_{t}){h}'(f(x_{t}))\\
\omega_2(x_t, x_i) &= \frac{(\lambda-1)^2}{2}(x_t-x_i)^T{f}''(x_{t})(x_t-x_i){h}'(f(x_{t}))\\
\omega_3(x_t, x_i) &= \frac{(\lambda-1)^2}{2}(x_t-x_i)^T{f}'(x_{t})(x_t-x_i)^T{f}'(x_{t}){h}''(f(x_{t})).
\end{align*}
\end{proposition}

\begin{proof}
    Let $\psi_t(\lambda)=h(f(x_{i\lambda}))$ which is modified function of $h(f(x_{i\lambda}))$ having $\lambda$ as an input. If $h(\cdot)$ and $f(\cdot)$ are twice differentiable with respect to each input. By the second-order Taylor approximation,
    \begin{equation}
        \psi_t(\lambda)=\psi_t(1)+{\psi}'_t(1)(\lambda-1)+\frac{1}{2}{\psi}''_t(1)(\lambda-1)^2+\varphi_t(\lambda)(\lambda-1)^2,
    \end{equation}
    where $\lim_{\lambda \rightarrow 1}\varphi_t(\lambda)=0$.
    \begin{equation*}
        {\psi}'_t(\lambda) = \frac{\partial x_{i\lambda}}{\partial \lambda}\frac{\partial f(x_{i\lambda})}{\partial x_{i\lambda}}\frac{\partial {h}(f(x_{i\lambda}))}{\partial f(x_{i\lambda})}
        = (x_t-x_i)^T{f}'(x_{i\lambda}){h}'(f(x_{i\lambda}))
    \end{equation*}
    \sloppy Since $\frac{\partial}{\partial \lambda}(x_t-x_i)^T{f}'(x_{i\lambda}){h}'(f(x_{i\lambda}))=\frac{\partial}{\partial \lambda}[(x_t-x_i)^T{f}'(x_{i\lambda})]{h}'(f(x_{i\lambda})) + (x_t-x_i)^T{f}'(x_{i\lambda})\frac{\partial}{\partial \lambda}[{h}'(f(x_{i\lambda}))] \;\text{and}\; \frac{\partial}{\partial \lambda}(x_t-x_i)^T{f}'(x_{i\lambda})=(x_t-x_i)^T{f}''(x_{i\lambda})(x_t-x_i)$,
    \begin{equation*}
        {\psi}''_t(\lambda) = (x_t-x_i)^T{f}''(x_{i\lambda})(x_t-x_i){h}'(f(x_{i\lambda})) + (x_t-x_i)^T{f}'(x_{i\lambda})(x_t-x_i)^T{f}'(x_{i\lambda}){h}''(f(x_{i\lambda}))
    \end{equation*}
    When $\lambda = 1$,
    \begin{align*}
        {\psi}'_t(1)&=(x_t-x_i)^T{f}'(x_{t}){h}'(f(x_{t}))\\
        {\psi}''_t(1)&=(x_t-x_i)^T{f}''(x_{t})(x_t-x_i){h}'(f(x_{t})) + (x_t-x_i)^T{f}'(x_{t})(x_t-x_i)^T{f}'(x_{t}){h}''(f(x_{t})).
    \end{align*}
    Fianlly, we derive Equation \ref{eq:proposition_eq1_app} in Proposition \ref{ood_func_mixup_app} as
    \begin{align}
        h(f(x_{i\lambda}))=h(f(x_t))
        &+(\lambda-1)(x_t-x_i)^T{f}'(x_{t}){h}'(f(x_{t}))\\
        &+\frac{(\lambda-1)^2}{2}(x_t-x_i)^T{f}''(x_{t})(x_t-x_i){h}'(f(x_{t}))\\ &+\frac{(\lambda-1)^2}{2}(x_t-x_i)^T{f}'(x_{t})(x_t-x_i)^T{f}'(x_{t}){h}''(f(x_{t}))\\
        &+\varphi_t(\lambda)(\lambda-1)^2. \nonumber
    \end{align}
\end{proof}

\section{Proof of Theorem \ref{theorem_mixdiff_app} and extension to other OOD scoring functions} \label{app:thm_proof}

\begin{theorem} \label{theorem_mixdiff_app}
Let h(x) represent MSP and f(x) represent a linear model, described by $w^Tx+b$, where $w, x \in \mathbb{R}^d$ and $b \in \mathbb{R}$. We consider the target sample, $x_t$, to be a hard OOD sample, defined as a sample that is predicted to be of the same class as the oracle sample, $x_m$, but with a higher confidence score than the oracle sample. For binary classification, $x_t$ is a hard OOD sample when $0 < f(x_m) < f(x_t) \text{ or } f(x_t) < f(x_m) < 0$. There exists an auxiliary sample $x_i$ such that
    \begin{equation}\label{eq:mixdiff_ineq}
        h(f(x_t))-h(f(x_m))+\sum_{l=1}^3(\omega_l(x_t, x_i)-\omega_l(x_m, x_i)) > 0.
    \end{equation}
\end{theorem}

\begin{proof}
    Considering \textup{MSP} in binary classification task, $\textup{MSP}=-\max(\sigma(f(x)), 1-\sigma(f(x)))$.  $f'(x)=w, \;\; f''(x)=\textbf{0}$,
    \begin{align*}
        h'(f(x)) &= \begin{cases}
                       -\sigma'(f(x)) & \text{ if } f(x)>0 \\ 
                       \sigma'(f(x)) & \text{    } \text{otherwise}
                    \end{cases} \\
        h''(f(x)) &= \begin{cases}
                        -\sigma''(f(x)) & \text{ if } f(x)>0 \\
                        \sigma''(f(x)) & \text{    } \text{otherwise}
                     \end{cases}
    \end{align*}
    where $\sigma(\cdot)$ denotes the sigmoid function. As the target sample is a hard OOD sample, it can be written as $f(x_t) = f(x_m)+c$ and $h(f(x_t)) < h(f(x_m))$ where $0< f(x_m) < f(x_t), 0 < c$ \:\text{and}\: $0.5 < \sigma(f(x_m)) < \sigma(f(x_t))$. Then, $h(f(x_t)) - h(f(x_m)) = -\sigma(f(x_t)) + \sigma(f(x_m))$. $-0.5 < -\sigma(f(x_t)) + \sigma(f(x_m)) < 0$.
    
    Equation \ref{eq:mixdiff_ineq} is equivalent to Equation \ref{eq:msp_mxdf_ineq} as $\omega_2 = 0$ under the assumption that $f(x)$ is a linear model.
    \begin{equation}\label{eq:msp_mxdf_ineq}
        h(f(x_t)) - h(f(x_m)) + (\omega_1(x_t, x_i)-\omega_1(x_m, x_i)) + (\omega_3(x_t, x_i) - \omega_3(x_m, x_i)) > 0
    \end{equation}
    \begin{align}
        \omega_1(x_t, x_i)-\omega_1(x_m, x_i)
        &= (\lambda-1)[(x_t-x_i)^Tw(-\sigma'(f(x_t)))-(x_m-x_i)^Tw(-\sigma'(f(x_m)))]\\
        &= (\lambda-1)[(f(x_i)-f(x_t))\sigma'(f(x_t))-(f(x_i)-f(x_m))\sigma'(f(x_m))] \\
        &= (\lambda-1)[(f(x_i)-f(x_t))\sigma'(f(x_t))-(f(x_i)-f(x_t)+c)\sigma'(f(x_m))]\\
        &= (\lambda-1)[(f(x_i)-f(x_t))(\sigma'(f(x_t))-\sigma'(f(x_m)))-c\sigma'(f(x_m))]. \label{eq:omega-1}
    \end{align}
    Because we assume $0 < f(x_m) < f(x_t)$, $\sigma'(f(x_t))-\sigma'(f(x_m)) < 0$, and $0 < \lambda < 1$. When $(\omega_1(x_t, x_i)-\omega_1(x_m, x_i)) > 0.5$,
    \begin{equation}
        f(x_i)\geq f(x_t)+\frac{(1/2(\lambda-1))+c\sigma'(f(x_m))}{\sigma'(f(x_t))-\sigma'(f(x_m))}
         \label{eq:omega1_fi}
    \end{equation}
    $f(x_i)$ denotes the confidence of the model with respect to auxiliary sample $x_i$. When $f(x_i)$ satisfies the above condition, Equation \ref{eq:msp_mxdf_ineq} holds when
    \begin{equation}
        \omega_3(x_t, x_i) - \omega_3(x_m, x_i) \geq 0.
    \end{equation}
    
    Let $\tau = \frac{{h}''(f(x_t))}{{h}''(f(x_m))} > 0$, then
    \begin{align}
         [(x_t-x_i)^Tw]^2{h}''(f(x_t))-[(x_m-x_i)^Tw]^2{h}''(f(x_m)) \geq 0
         \\ \label{eq:omega-3-1}
        [(f(x_t)-f(x_i))^2\tau-(f(x_m)-f(x_i))^2]{h}''(f(x_m)) \geq 0.
    \end{align}
    Because of ${{h}''(f(x_m)) > 0},$
    \begin{align}
        (f(x_t)-f(x_i))^2\tau-(f(x_m)-f(x_i))^2 \geq 0
        \\
        (f(x_t)-f(x_i))^2\tau-(f(x_t)-c-f(x_i))^2 \geq 0.
    \end{align}
    Let $t = f(x_t)-f(x_i)$, then
    \begin{align}
        t^2\tau - (t-c)^2 = (\tau-1)t^2+2ct-c^2. \label{eq:quadratic}
    \end{align}

    By reformulating the Equation \ref{eq:quadratic} with respect to $f(x_i)$, we obtain the following expression.
    \begin{align}
        (\tau-1)f(x_i)^2-2((\tau-1)f(x_t)+c)f(x_i)+(\tau-1)f(x_t)^2+2cf(x_t)-c^2 \geq 0. \label{eq:ineq_fi}
    \end{align}
    When $0<\tau$, the discriminant of the Equation \ref{eq:ineq_fi} with respect to $f(x_i)$ is positive and the value of the right side of \ref{eq:omega1_fi} exists between the two solution values for which \ref{eq:ineq_fi} equals zero with respect to $f(x_i)$.
\end{proof}

\begin{theorem}
Theorem \ref{theorem_mixdiff_app} holds for Entropy OOD scoring function.
\end{theorem}

\begin{proof}
    \sloppy Considering \textup{Entropy} OOD score function in binary classification task, \textup{Entropy} = $-(\sigma (f(x)) \log (\sigma (f(x))) + (1-\sigma (f(x))) \log (1-\sigma (f(x))))$. $f'(x) = w$ and $f''(x) = \textbf{0}$.
    
    Let us express the scores of a hard OOD sample and an oracle sample as $f(x_t), f(x_m) > 0, f(x_t) = f(x_m) + c, c>0$. Then, $-\epsilon < h(f(x_t)) - h(f(x_m)) < 0$, where $-\epsilon < 0$ denotes the lower bound of the difference between the OOD scores of the target and oracle samples. Followed by Equation \ref{eq:omega-1},
    \begin{equation*}
        (\lambda -1)(f(x_t)-f(x_m))(h'(f(x_t))-h'(f(x_m)))+(\lambda-1)ch'(f(x_m)).
    \end{equation*}
    Because the sign of $h'(f(x_m))$ is a negative when $f(x_m) > 0$, $(\lambda - 1) c h'(f(x_m)) \geq 0$. $h(f(x_t))-h(f(x_m))+(\omega_1(x_t, x_i)-\omega_1(x_m, x_i)) \geq 0$, where $(\lambda - 1)(f(x_t)-f(x_m))(h'(f(x_t))-h'(f(x_m))) \geq \epsilon$.
    \begin{align}
        f(x_i) \geq f(x_t)-\frac{\epsilon}{(\lambda -1)(h'(f(x_t))-h'(f(x_m)))}, & \:\text{ if }\; h'(f(x_t))-h'(f(x_m))>0 \\
        f(x_i) \leq f(x_t)-\frac{\epsilon}{(\lambda -1)(h'(f(x_t))-h'(f(x_m)))}, & \:\text{ if }\; h'(f(x_t))-h'(f(x_m))<0 \label{eq:f-xi}
    \end{align}   
    Under the assumption that $f(x_i)$ satisfies the above condition, Equation \ref{eq:mixdiff_ineq} holds when
    \begin{equation}
        \omega_3(x_t, x_i)-\omega_3(x_m, x_i) \geq 0. \label{eq:omega-3}
    \end{equation}

    Let $\tau = \frac{h''(f(x_m))}{h''(f(x_t))} > 0$, then we follow the same steps in Equation \ref{eq:omega-3-1} - Equation \ref{eq:quadratic}. There exists $x_i$ such that it satisfies Equation \ref{eq:omega-3} and Equation \ref{eq:f-xi}.
\end{proof}

\begin{theorem}
Theorem \ref{theorem_mixdiff_app} holds for MLS OOD scoring function.
\end{theorem}

\begin{proof}
    Considering \textup{MLS} OOD score function in binary classification task, $\textup{MLS} = -f(x)$. Equation \ref{eq:mixdiff_ineq} is equivalent to Equation \ref{eq:mls_mxdf_ineq} as $\omega_2 = \omega_3 = 0$ because $f''(x) = h''(f(x)) = 0$.

    \begin{equation}\label{eq:mls_mxdf_ineq}
        h(f(x_t)) - h(f(x_m)) + \omega_1(x_t, x_i) - \omega_1(x_m, x_i) > 0
    \end{equation}
    
    The right hand-side of Equation \ref{eq:mls_mxdf_ineq} is written as
    \begin{align}
        &-f(x_t) + f(x_m) + (1-\lambda)[(f(x_t)-f(x_i))-(f(x_m)-f(x_i))]\\
        =&-f(x_t)+f(x_m)+(1-\lambda)(f(x_t)-f(x_m))\\
        =&-\lambda f(x_t) + \lambda f(x_m).
    \end{align}
    If $x_t$ is an OOD sample and $f(x_m) > f(x_t)$ where $f(x_t), f(x_m) > 0$, Equation \ref{eq:mls_mxdf_ineq} holds.
\end{proof}

\section{Experimental validation of Proposition \ref{ood_func_mixup_app} and Theorem \ref{theorem_mixdiff_app}}

\begin{figure}[h]
\begin{subfigure}[b]{0.49\textwidth}
    \centering
    \includegraphics[width=\textwidth]{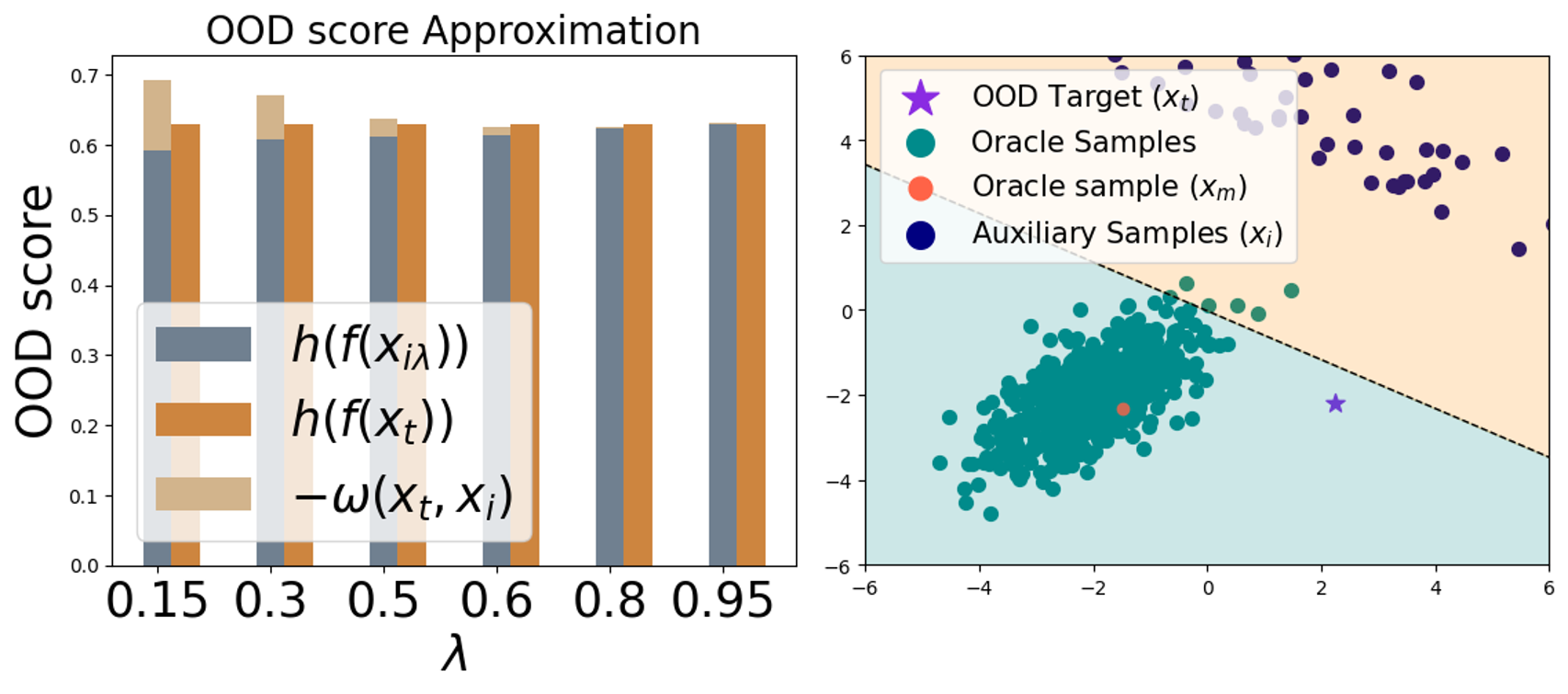}
    \caption{MixDiff+Entropy}
    \label{fig:entropy_theory}
\end{subfigure}\hfill
\begin{subfigure}[b]{0.49\textwidth}
    \centering
    \includegraphics[width=\textwidth]{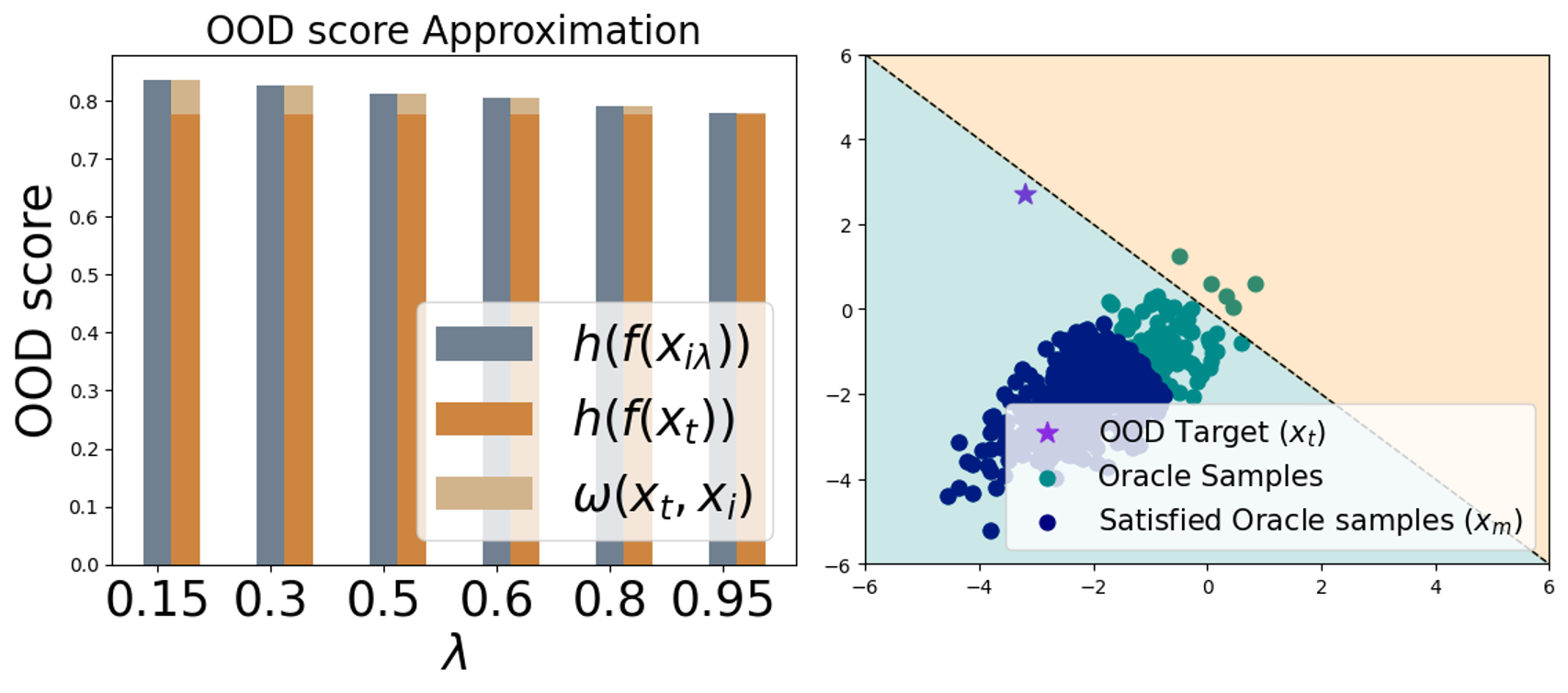}
    \caption{MixDiff+MLS}
    \label{fig:mls_theory}
\end{subfigure}
\caption{On the left of (a) and (b), the OOD scores of a mixed target sample are compared with the approximated OOD scores. On the right of (a) and (b), auxiliary samples are shown along with the oracle samples that guarantee MixDiff is positive.}
\end{figure}

We experimentally verify that Proposition \ref{ood_func_mixup_app} and Theorem \ref{theorem_mixdiff_app} hold when the base OOD score functions are Entropy and MLS, respectively. We conduct verification experiments on a synthetic dataset consisting of 2-dimensional features following the same setup as in Section 3.1 of the main paper. On the right side of Figure \ref{fig:entropy_theory}, we plot the OOD target sample, an oracle sample that has the same class as the predicted class of the target sample, and auxiliary samples that satisfy the condition that makes MixDiff positive. On the right side of Figure \ref{fig:mls_theory}, we show a single target sample along with the oracle samples that satisfy the condition for having a positive MixDiff score. The assumption of binary classification with a linear model eliminates the effect of auxiliary samples.

\section{Verification experiment of the main motivation}\label{sec:veri_exp}

Our primary hypothesis is that overemphasized features are more susceptible to perturbations compared to the features that actually belong to the predicted class. To test this hypothesis, we utilize class activation map (CAM) \citep{Chen_2022_ACCV} to observe the changes in the model's attention areas before and after Mixup operation.

\begin{figure}[t]
    \centering
    \includegraphics[width=0.6\textwidth]{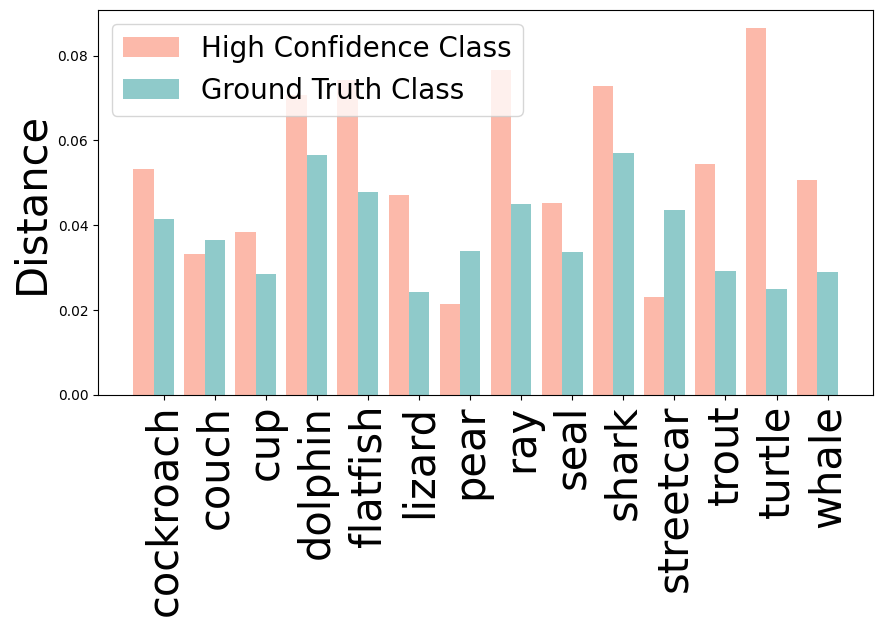}
    \caption{The average pixel-wise difference of the CAM images is measured before and after mixing the OOD samples with an auxiliary sample. The high confidence class and the ground truth class represent the classes used for the prompts in CAM. We measure the fluctuations in the areas of the model focus when the OOD samples are perturbed by arbitrary signals such as mixing with an auxiliary sample.}
    \label{fig:cam_diff}
\end{figure}

We conduct the experiment using CLIP ViT-B/32 and follow the same settings as in the OOD detection experiments on CIFAR100. We first collect OOD samples that are misclassified as ID by the MSP score with a threshold set by TPR95. This set contains samples for which the model has exhibited high confidence. We then filter these samples to include only those classes with at least five samples per class. For each sample, an auxiliary sample for Mixup was randomly selected from an ID class, excluding the class with the highest confidence for that sample.

We compare the CAMs of the high-confidence OOD samples before and after Mixup. The CAMs are processed through min-max normalization, and values below 0.8 are clipped to be zero. We then measure the $L_1$ distance between the two CAMs. To observe the difference in CAMs before and after Mixup for the predicted high-confidence class, we use the text prompt of the class predicted by the model. Similarly, we measure the difference in CAMs of the ground truth classes under Mixup operation with the text prompt of the ground truth class of the sample. The use of ground truth class is to eliminate the effect of sample-wise scale differences. For example, some samples have a large or small object area compared to others.

Figure \mbox{\ref{fig:cam_diff}} compares the average distance in CAMs for each class, considering the prompts as either a high confidence class or a ground truth class. A smaller distance implies less variation due to perturbation, suggesting that the features that the model focuses on are highly relevant to the respective class. On the other hand, a larger distance indicates a greater variation due to perturbation, which could mean that the features that the model focuses on are either less relevant to the class or incorrectly identified as relevant features.
The results in Figure \mbox{\ref{fig:cam_diff}} indicate how perturbations can be used to assess the reliability of the features that lead to a high level of confidence in the input predictions of the model.

\section{Practical implementation}
For each target sample $x_t$, MixDiff generates $N \times R$ mixed samples. Similarly, it generates $N \times R$ mixed samples for each of $M$ oracle samples. If we follow the in-batch setup where the samples that are in the same batch as the target sample are used as the auxiliary samples, MixDiff requires processing of $BNR + BMNR$ mixed samples, denoting the batch size as $B = N+1$.

We avoid $BNR + BMNR$ repeated forward passes by putting each set of the entire Mixup results, including the ones that are mixed with itself, into two tensors of sizes that are prefixed with $(B, B, R)$ and $(B, M, B, R)$, one for the mixed images of targets and auxiliary samples, the other for the mixed images of oracles and auxiliary samples, respectively. After computing the yet-to-be-averaged MixDiff scores within a tensor of size $(B, B, R)$, we zero out the diagonal entries in the first two dimensions, $(B, B)$, eliminating the scores from the target images that are mixed with itself. Then, we take the average of the last two dimensions, $(B, R)$, yielding $B$ MixDiff scores for each of the $B$ test samples.

We also note that in practice the set of $\Bar{O}^{*}_{ir}$ prediction scores corresponding to the oracle samples mixed with the other in-batch samples do not need to be computed for every single test batch. One can use a fixed set of samples as an auxiliary set and precompute each of the mixed oracle logits $\Bar{O}^{*}_{ir}$ by mixing these samples with the oracle samples. When a test batch arrives, each of the samples in the batch will then be independently mixed with these fixed auxiliary samples. Not only does it reduce the compute cost, there is no dependency on the test batch size in regard to OOD detection performance, since the auxiliary samples are no longer drawn from the test batch.

\section{Experimental details} \label{app:exp_detail}
\subsection{Experimental setup}
We evaluate MixDiff within the setting where the class names of OOD samples and the OOD labels are unavailable at train time. This is a more challenging experimental setting compared to the environment where the OOD class names or its instances are known during the training phase. We follow the same setup as in \citet{esmaeilpour2022zero}, and evaluate our method on five OOD detection benchmark datasets: CIFAR10 \citep{krizhevsky2009learning}, CIFAR100 \citep{krizhevsky2009learning}, CIFAR+10 \citep{miller2021class}, CIFAR+50 \citep{miller2021class}, TinyImageNet \citep{le2015tiny}.

Each dataset's ID and OOD (known and unknown) class splits are constructed as follows.
\textbf{CIFAR10}: the dataset's 10 classes are randomly split into 6 ID classes and 4 OOD classes.
\textbf{CIFAR100}: consecutive 20 classes are assigned to be ID classes and the remaining 80 classes are assigned to be OOD classes.
\textbf{CIFAR+10}: 4 non-animal classes of CIFAR10 are ID classes, 10 randomly sampled animal classes from CIFAR100 are OOD classes.
\textbf{CIFAR+50}: 4 non-animal classes of CIFAR10 are ID classes, 50 randomly sampled animal classes from CIFAR100 are OOD classes.
\textbf{TinyImageNet}: considers 20 randomly sampled classes as ID classes and the remaining 180 classes as OOD classes.

For CIFAR10, CIFAR+10, CIFAR+50 and TinyImageNet, we follow the same ID, OOD class splits as in \citet{miller2021class, esmaeilpour2022zero}. For CIFAR100, we use the same class splits as in \citet{esmaeilpour2022zero}. Each dataset contains 5 splits, except for CIFAR+50, which is consisted of only one ID, OOD class split. Figure \ref{fig:five_datasets} shows each method's average AUROC scores averaged over the five datasets. The setup takes logits as model outputs. All of the results for non-training-free methods are from \citet{esmaeilpour2022zero} except for ZOC and MixDiff+ZOC.

We utilize CLIP's \citep{radford2021learning} zero-shot classification capability for OOD detection. More specifically, we compute the similarity score for each ID class label's prompt, \texttt{"This is a photo of \{label\}"}, with the target image, and use these as logits. Since OOD class labels are not known a priori, this forms a valid experimental setup even though the CLIP model is performing zero-shot classification task \citep{esmaeilpour2022zero, mingdelving, wang2023clipn}.

\begin{wrapfigure}{R}{0.4\textwidth}
    \centering
    \includegraphics[width=1.0\linewidth]{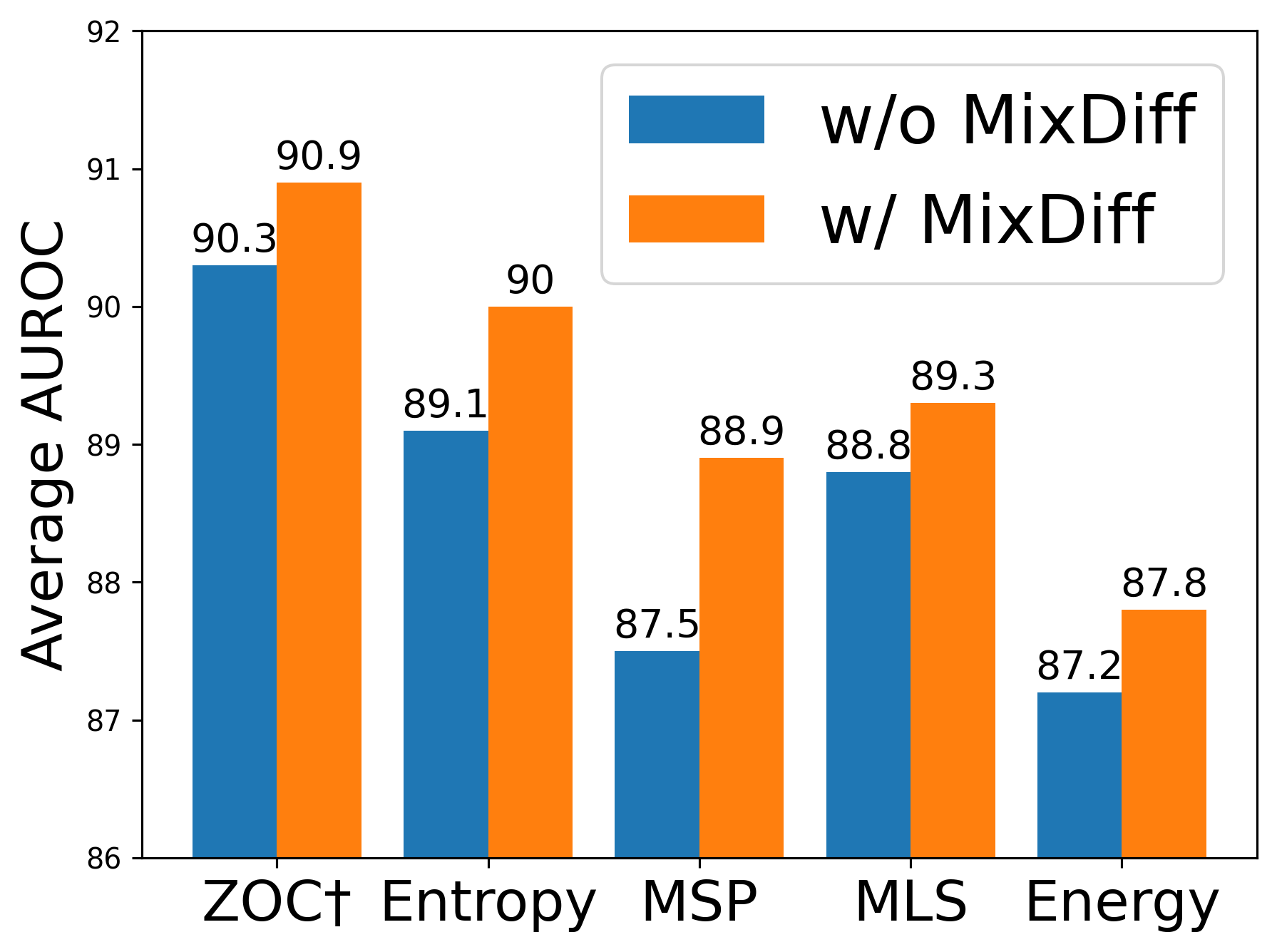}
    \caption{AUROC scores averaged over the five datasets.}
    \label{fig:five_datasets}
\end{wrapfigure}

\subsection{Evaluation metrics}
We compare our method with the baseline methods using the metrics that are commonly employed for OOD detection tasks. \textbf{AUROC} denotes area under the receiver operating characteristic where the receiver operating characteristic represents the relationship between false positive rate (FPR) and true positive rate (TPR) for all of the threshold range. \textbf{FPR95} denotes the false positive rate when the threshold satisfies 95\% TPR. \textbf{AUCPR} represents area under the curve of precision and recall. It is a useful performance measure, especially with an imbalanced dataset. For AUCPR, we set the detection threshold to be the value that satisfies 95\% TPR. We consider OOD samples as positive.

\subsection{Hyperparameter search on Caltech101}
We construct each known-unknown class split for Caltech101 dataset \citep{fei2004learning} by randomly sampling 20 classes as ID, and setting aside the rest as OOD, making a total of 3 splits. We conduct grid search over the following hyperparameter configurations: $M \in \{15, 10\}$, $N \in \{14, 9\}$, $R \in \{7, 5\}$, $\gamma \in \{2.0, 1.0, 0.5\}$. We use the numbers that evenly divide the interval $[0, 1]$ into $R+1$ segments as the values of the Mixup ratios. For example, when $R = 3$, the set of Mixup ratios is $\{0.25, 0.5, 0.75\}$. We select the configuration with the highest average AUROC score for each method. For the environment where the model outputs are the logits, the resulting hyperparameters are $M=15$, $N=14$, $R=7$, and $\gamma=2$ for all methods.

\sloppy For Entropy+MSP linear combination baseline, we tune the scaling factor $\eta = b \times 10^{a}$, by conducting grid search over the following configurations: $a \in \{ -4, -3, -2, -1, 0, 1, 2, 3, 4\}$, $b \in \{1, 2, 3, 4, 5, 6, 7, 8, 9\}$ and the score to which $\eta$ is multiplied (MSP or Entropy).

We conduct hyperparameter search for ASH \citep{djurisic2023extremely} over the pruning percentile $p \in \{ 10, 20, 30, 40, 50, 60, 70, 80, 90\}$ and 3 treatment methods of unpruned activations, namely, ASH-P, ASH-B, ASH-S, on Caltech101. The same CLIP ViT-B/32 \citep{radford2021learning} backbone is employed for zero-shot classification and the entropy score is utilized as the OOD scoring function.

We conduct hyperparameter search over the DML's \citep{Zhang_2023_CVPR} scaling ratio $\lambda \in \{ 0.01, 0.1, 1.0, 2.0, 5.0, 10.0, 30.0, 60.0, 100.0, 300.0, 500.0, 1000.0 \}$ on the Caltech101 and use the best performing value in terms of AUROC when evaluating on the other datasets. The same CLIP ViT-B/32 \citep{radford2021learning} backbone is utilized without any finetuning on ID samples.

\subsection{Adaptation of MixDiff with ZOC}\label{app:zoc}
ZOC \citep{esmaeilpour2022zero} utilizes a candidate OOD class name generator. MixDiff framework is applied to ZOC by averaging out each of the perturbed images' candidate OOD logits as follows: $ \log (\frac{1}{C} \sum_{i=1}^{C} \exp{(o_i)})$ where $C$ and $o_i$ are the number of generated OOD class names from the image and the $i$-th OOD class logit, respectively. This effectively means that the logits in the perturbed oracle and target samples' outputs have a dimension of $K+1$ instead of $K$ in the following equations: $O^{*}_{mir} = f(x^{*}_{mir}) \in \mathbb{R}^K$ and $O_{ir} = f(x_{ir}) \in \mathbb{R}^K$. 

We evaluate on 200 randomly chosen samples per split as ZOC’s token generation module requires a large amount of computation to process the entire set of mixed images. Also, the hyperparameters were tuned on each of the target datasets to alleviate variability issues.

\subsection{Experimental details on adversarial defence task} \label{sec:adv_detail}
We take the same experimental setup as the OOD detection experiments with identical datasets and backbone model. We use projected gradient descent (PGD) attack \mbox{\citep{madry2018towards}} with $L_\infty$ norm perturbation bound, adversarial budget $\epsilon = \frac{1}{255}$ and attack step size of 10. We assume access to the model parameters for the attacker, so that the true gradients can be calculated. Following \mbox{\citet{chen2022robust}}, cross entropy with the uniform distribution is used as the loss function when attacking ID samples, and Shannon entropy is used as the loss function when attacking OOD samples. We use the other oracle samples that are of the same class as the predicted label of the target as auxiliary samples, referred to as oracle as auxiliary in the main paper. We use the same hyperparameters that are found in OOD detection task without separate hyperparameter tuning on adversarial defence task. Figure \ref{fig:adv_step_size} shows AUROC scores for various attack step sizes.

\begin{figure}
\begin{subfigure}[b]{1.0\textwidth}
\centering
\resizebox{0.4\textwidth}{!}{%
    \begin{tikzpicture}
    \begin{axis}[
        xlabel={Attack Step Size},
        ylabel={AUROC},
        ymax=85,
        ymin=0,
        legend pos=north east,
        ymajorgrids=true,
        grid style=dashed,
    ]
    
    \addplot[
        color=blue,
        mark=*,
        ]
        coordinates {
        (0,79.87)
        (1,29.26)
        (5,5.236)
        (10,2.213)
        (20,1.712)
        (50,1.371)
        (100,1.34)
        };
        \addlegendentry{Entropy}
        
    \addplot[
        color=orange,
        mark=*,
        ]
        coordinates {
        (0,81.11)
        (1,39.34)
        (5,14.95)
        (10,9.084)
        (20,8.074)
        (50,7.467)
        (100,7.569)
        };
        \addlegendentry{MixDiff+Entropy}
        
    \addplot[
        color=red,
        mark=*,
        ]
        coordinates {
        (0,78.05)
        (1,49.54)
        (5,34.16)
        (10,27.48)
        (20,27.02)
        (50,26.57)
        (100,26.79)
        };
        \addlegendentry{MixDiff Only}
    \end{axis}
    \end{tikzpicture}
}
\caption{Both ID and OOD samples are attacked.}
\label{fig:adv_step_size_both}
\end{subfigure}

\hspace{0.25cm}

\begin{subfigure}[b]{1.0\textwidth}
\centering
\resizebox{0.4\textwidth}{!}{%
    \begin{tikzpicture}
    \begin{axis}[
        xlabel={Attack Step Size},
        ylabel={AUROC},
        ymax=85,
        ymin=0,
        legend pos=north east,
        ymajorgrids=true,
        grid style=dashed,
    ]
    
    \addplot[
        color=blue,
        mark=*,
        ]
        coordinates {
        (0,79.87)
        (1,60.32)
        (5,44.9)
        (10,36.86)
        (20,34.33)
        (50,33.4)
        (100,33.78)
        };
        \addlegendentry{Entropy}
        
    \addplot[
        color=orange,
        mark=*,
        ]
        coordinates {
        (0,81.11)
        (1,66.21)
        (5,56.73)
        (10,50.31)
        (20,49.13)
        (50,48.77)
        (100,49.25)
        };
        \addlegendentry{MixDiff+Entropy}
        
    \addplot[
        color=red,
        mark=*,
        ]
        coordinates {
        (0,78.05)
        (1,67.65)
        (5,63.1)
        (10,58.84)
        (20,58.58)
        (50,58.69)
        (100,58.98)
        };
        \addlegendentry{MixDiff Only}
    \end{axis}
    \end{tikzpicture}
}
\caption{ID samples are attacked.}
\label{fig:adv_step_size_id2ood}
\end{subfigure}

\hspace{0.25cm}

\begin{subfigure}[b]{1.0\textwidth}
\centering
\resizebox{0.4\textwidth}{!}{%
    \begin{tikzpicture}

    \begin{axis}[
        xlabel={Attack Step Size},
        ylabel={AUROC},
        ymax=85,
        ymin=0,
        legend pos=north east,
        ymajorgrids=true,
        grid style=dashed,
    ]
    
    \addplot[
        color=blue,
        mark=*,
        ]
        coordinates {
        (0,79.87)
        (1,51.95)
        (5,22.31)
        (10,14.38)
        (20,12.47)
        (50,11.08)
        (100,10.76)
        };
        \addlegendentry{Entropy}
        
    \addplot[
        color=orange,
        mark=*,
        ]
        coordinates {
        (0,81.11)
        (1,56.92)
        (5,36.55)
        (10,31.4)
        (20,30.48)
        (50,29.72)
        (100,29.46)
        };
        \addlegendentry{MixDiff+Entropy}
        
    \addplot[
        color=red,
        mark=*,
        ]
        coordinates {
        (0,78.05)
        (1,59.37)
        (5,47.68)
        (10,44.19)
        (20,43.8)
        (50,43.4)
        (100,43.17)
        };
        \addlegendentry{MixDiff Only}
    \end{axis}
    \end{tikzpicture}
}
\caption{OOD samples are attacked.}
\label{fig:adv_step_size_ood2id}
\end{subfigure}

\caption{AUROC scores for various attack step sizes on CIFAR100. \textbf{(a)} Both ID and OOD samples are adversarially attacked.  \textbf{(b)} ID samples are adversarially attacked. \textbf{(c)} OOD samples are adversarially attacked.}
\label{fig:adv_step_size}
\end{figure}
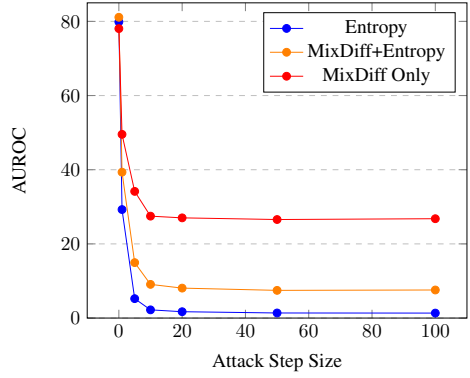
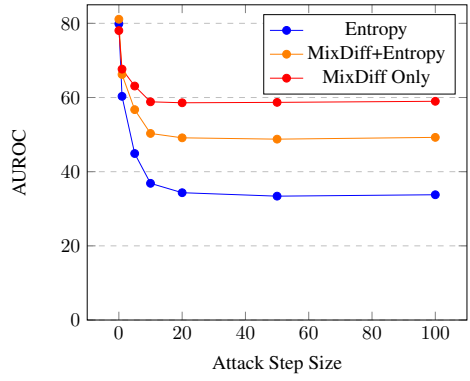
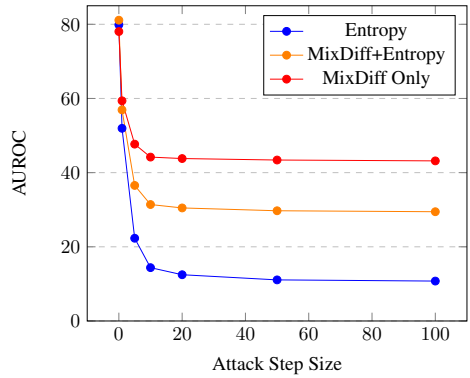

\subsection{Experimental details on out-of-scope detection task} \label{sec:text_mixup}
We run out-of-scope detection experiments using 4 intent classification datasets. CLINC150 \citep{larson-etal-2019-evaluation} dataset is consisted of samples spanning across 10 general domains including \texttt{"utility"} and \texttt{"travel"}, with each sample belonging to one of 150 intent classes. Banking77 \citep{Casanueva2020} is a dataset specializing in banking domain and has 77 intent classes. ACID \citep{acharya2020using} is an intent detection dataset with 175 intents, consisted of samples of customers contacting an insurance company. TOP \citep{gupta-etal-2018-semantic-parsing} is a dataset with the intents organized in a hierarchical structure and is consisted of queries related to navigation and event. For TOP dataset, we use the root node's intent as the intent label for the query, as in \citet{yilmaz-toraman-2022-d2u}.

\sloppy For CLINC150 and TOP datasets, we keep the original OOS intents in the OOS split. More specifically, CLINC150 dataset's \texttt{"oos"} class and TOP dataset's intent classes that are prefixed with \texttt{"IN:UNSUPPORTED"} \citep{yilmaz-toraman-2022-d2u}. We also set aside 4 intents in TOP dataset that have too small number of samples to be reliably split into train and validation sets as OOS intents. These intents are \texttt{"IN:GET\_EVENT\_ATTENDEE"}, \texttt{"IN:UNINTELLIGIBLE"}, \texttt{"IN:GET\_EVENT\_ORGANIZER"}, and \texttt{"IN:GET\_EVENT\_ATTENDEE\_AMOUNT"}. This leaves the dataset with 12 original in-scope intent classes, excluding the OOS intent classes.
 
We further split the train set of the in-scope samples into more in-scope, OOS splits and use these to search MixDiff's hyperparameters. To assume an environment where the test time in-scope ratio is unknown, we evaluate OOS detection performance on multiple inner in-scope ratios, 25\%, 50\%, 75\%, for each inner split. We leave out the splits with the number of inner in-scope intents less than 2. An intent classification model is trained for each of these inner in-scope splits. After training, we perform OOS detection on the outer in-scope validation set and select the hyperparameter set with the highest average AUROC score.

For a given oracle sample, we use the other oracle samples in the same class as the auxiliary samples. For ease of comparison between the logit-based and probability-based OOD scoring functions, we take the setup where the model $f(\cdot)$'s outputs are in the logit space for both cases.

We explore three configurations with respect to the position of the auxiliary sample in a concatenated text pair: (1) prepending the auxiliary sample at the front of an oracle or the target sample; (2) appending the auxiliary sample at the end of an oracle or the target sample; (3) a combination of both, analogous to the setting of 2 Mixup ratios in image Mixup ($R=2$). We conduct grid search over the following hyperparameters: $M \in \{ 5, 10, 15, 20, 25, 30 \}$, $\gamma \in \{ 0.5, 1.0, 2.0\}$, and three auxiliary sample concatenation methods as described above. We note that the number of auxiliary samples is determined as $N = M - 1$, since we use the other oracle samples in the same class as the auxiliary samples. We provide the average AUROC scores for each in-score ratio in Table \ref{tab:text_auroc_full}.

\begin{table}[H]
\centering
\resizebox{1.0\textwidth}{!}
{
\begin{tabular}{lcccccc}
\toprule
Method                                                      & In-scope ratio    & CLINC150  
                                                                                                           & Banking77                 & ACID                        & TOP                            & Average        \\ \midrule
\multirow{4}{*}{MSP \citep{hendrycks2017a}}                  & 25\%             & 93.07\scriptsize{$\pm1.8$} & 84.29\scriptsize{$\pm3.7$} &  89.39\scriptsize{$\pm1.6$}  & 93.68\scriptsize{$\pm4.5$}  & 90.11  \\ 
                                                            & 50\%             & 93.26\scriptsize{$\pm0.6$} & 85.80\scriptsize{$\pm$3.2} &  88.61\scriptsize{$\pm$1.3}  & 88.90\scriptsize{$\pm$5.2}  & 89.14  \\
                                                            & 75\%             & 92.74\scriptsize{$\pm0.8$} & 86.20\scriptsize{$\pm$3.6} &  88.93\scriptsize{$\pm$1.8}  & 87.44\scriptsize{$\pm$7.0}  & 88.83  \\ 
                                                            & Avg.             & 93.02                      & 85.43                      &  88.98                       & 90.01                       & 89.36  \\ \midrule
                                                            
\multirow{4}{*}{MLS \citep{Hendrycks2022ScalingOD}}          & 25\%             & 93.06\scriptsize{$\pm$2.0} & 83.01\scriptsize{$\pm$3.8} &  88.96\scriptsize{$\pm$1.4}  & 93.10\scriptsize{$\pm$4.5}  & 89.53  \\
                                                            & 50\%             & 93.77\scriptsize{$\pm$0.6} & 85.63\scriptsize{$\pm$3.2} &  88.77\scriptsize{$\pm$1.0}  & 88.30\scriptsize{$\pm$6.2}  & 89.12  \\
                                                            & 75\%             & 93.85\scriptsize{$\pm$0.8} & 86.43\scriptsize{$\pm$3.8} &  89.00\scriptsize{$\pm$1.5}  & 88.77\scriptsize{$\pm$6.1}  & 89.51  \\ 
                                                            & Avg.             & 93.56                      & 85.02                      &  88.91                       & 90.06                       & 89.39  \\ \midrule
                          
\multirow{4}{*}{Energy \citep{liu2020energy}}                & 25\%             & 93.09\scriptsize{$\pm$2.1} & 82.96\scriptsize{$\pm$3.8} &  88.87\scriptsize{$\pm$1.4}  & 93.10\scriptsize{$\pm$4.5}  & 89.51  \\
                                                            & 50\%             & 93.82\scriptsize{$\pm$0.6} & 85.64\scriptsize{$\pm$3.2} &  88.70\scriptsize{$\pm$1.0}  & 88.30\scriptsize{$\pm$6.2}  & 89.12  \\
                                                            & 75\%             & 93.91\scriptsize{$\pm$0.8} & 86.36\scriptsize{$\pm$3.8} &  88.93\scriptsize{$\pm$1.5}  & 88.78\scriptsize{$\pm$6.1}  & 89.50  \\ 
                                                            & Avg.             & 93.61                      & 84.99                      &  88.83                       & 90.06                       & 89.37  \\ \midrule
                                                            
\multirow{4}{*}{Entropy \citep{Thulasidasan2021AnEB}}  & 25\%             & 93.23\scriptsize{$\pm$1.8} & 84.28\scriptsize{$\pm$3.8} &  89.27\scriptsize{$\pm$1.6}  & 93.68\scriptsize{$\pm$4.5}  & 90.12  \\
                                                            & 50\%             & 93.52\scriptsize{$\pm$0.6} & 86.02\scriptsize{$\pm$3.3} &  88.53\scriptsize{$\pm$1.2}  & 88.91\scriptsize{$\pm$5.2}  & 89.25  \\
                                                            & 75\%             & 93.11\scriptsize{$\pm$0.8} & 86.48\scriptsize{$\pm$3.8} &  88.81\scriptsize{$\pm$1.7}  & 87.46\scriptsize{$\pm$7.0}  & 88.97  \\ 
                                                            & Avg.             & 93.29                      & 85.59                      &  88.87                       & 90.02                       & 89.44  \\ \midrule
                                                            
\multirow{4}{*}{MixDiff+MSP}                                & 25\%             & 93.57\scriptsize{$\pm$1.7} & 84.77\scriptsize{$\pm$3.6} &  89.66\scriptsize{$\pm$1.6}  & 93.68\scriptsize{$\pm$4.6}  & \underline{90.42}  \\
                                                            & 50\%             & 93.57\scriptsize{$\pm$0.6} & 86.11\scriptsize{$\pm$3.0} &  88.77\scriptsize{$\pm$1.2}  & 89.65\scriptsize{$\pm$4,6}  & \underline{89.53}  \\
                                                            & 75\%             & 93.12\scriptsize{$\pm$0.8} & 86.36\scriptsize{$\pm$3.4} &  89.10\scriptsize{$\pm$1.6}  & 88.71\scriptsize{$\pm$6.1}  & 89.32  \\  
                                                            & Avg.             & 93.42                      & \underline{85.75}              &  \underline{89.18}               & \textbf{90.68}              & \underline{89.76}  \\ \midrule
                    
\multirow{4}{*}{MixDiff+MLS}                                & 25\%             & 93.57\scriptsize{$\pm$2.0} & 83.56\scriptsize{$\pm$3.7} &  89.37\scriptsize{$\pm$1.4}  & 93.16\scriptsize{$\pm$4.4}  & 89.92  \\
                                                            & 50\%             & 94.02\scriptsize{$\pm$0.6} & 86.02\scriptsize{$\pm$3.3} &  89.01\scriptsize{$\pm$1.0}  & 88.84\scriptsize{$\pm$5.8}  & 89.47  \\
                                                            & 75\%             & 94.04\scriptsize{$\pm$0.7} & 86.81\scriptsize{$\pm$3.6} &  89.33\scriptsize{$\pm$1.3}  & 89.04\scriptsize{$\pm$6.0}  & \textbf{89.81}  \\
                                                            & Avg.             & \underline{93.88}              & 85.46                      &  \textbf{89.24}              & \underline{90.35}               & 89.73  \\ \midrule
                          
\multirow{4}{*}{MixDiff+Energy}                             & 25\%             & 93.59\scriptsize{$\pm$2.0} & 83.51\scriptsize{$\pm$3.7} &  89.28\scriptsize{$\pm$1.4}  & 93.17\scriptsize{$\pm$4.4}  & 89.89  \\
                                                            & 50\%             & 94.01\scriptsize{$\pm$0.6} & 86.27\scriptsize{$\pm$2.9} &  88.95\scriptsize{$\pm$1.0}  & 88.83\scriptsize{$\pm$5.8}  & 89.52  \\
                                                            & 75\%             & 94.07\scriptsize{$\pm$0.8} & 86.74\scriptsize{$\pm$3.7} &  89.32\scriptsize{$\pm$1.3}  & 89.05\scriptsize{$\pm$6.0}  & \underline{89.80}  \\
                                                            & Avg.             & \textbf{93.89}             & 85.51                      &  \underline{89.18}               & \underline{90.35}               & 89.73  \\ \midrule
                      
\multirow{4}{*}{MixDiff+Entropy}                            & 25\%             & 93.70\scriptsize{$\pm$1.7} & 84.79\scriptsize{$\pm$3.7} &  89.55\scriptsize{$\pm$1.6}  & 93.70\scriptsize{$\pm$4.5}  & \textbf{90.44}  \\
                                                            & 50\%             & 93.84\scriptsize{$\pm$0.6} & 86.42\scriptsize{$\pm$3.1} &  88.74\scriptsize{$\pm$1.2}  & 89.65\scriptsize{$\pm$4.7}  & \textbf{89.66}  \\
                                                            & 75\%             & 93.48\scriptsize{$\pm$0.8} & 86.74\scriptsize{$\pm$3.6} &  89.09\scriptsize{$\pm$1.5}  & 88.68\scriptsize{$\pm$6.2}  & 89.50  \\ 
                                                            & Avg.             & 93.67                      & \textbf{85.98}             &  89.13                       & \textbf{90.68}              & \textbf{89.87}  \\ \bottomrule
                        
\end{tabular}%
}
\caption{Average AUROC scores for out-of-scope detection task. The numbers on the right side of $\pm$ represent standard deviation. The numbers in the "Average" column are the average AUROC scores reported in that row. The numbers in a "Avg." row are the average of the AUROC scores reported in that column. The highest and second highest average AUROC scores are highlighted with \textbf{bold} and \ul{underline}, respectively.}
\label{tab:text_auroc_full}
\end{table}

\section{Performance evaluation with AUCPR and FPR95}
Table \ref{tab:oth_metrics} presents a comprehensive performance analysis of MixDiff in relation to other baselines, utilizing commonly employed metrics for OOD detection studies. Our findings show that MixDiff can boost OOD detection performance in FPR95 and AUCPR as well as AUROC.

\begin{table}[H]
\centering
\resizebox{0.7\columnwidth}{!}
{%
\begin{tabular}{lcccc} \toprule
Method                      & Training-free & AUROC ($\uparrow$)                     & FPR95 ($\downarrow$) & AUCPR ($\uparrow$) \\ \midrule
ZOC $^\dagger$              & \ding{55}     & 82.7\scriptsize{$\pm$2.8}             & \textbf{64.0}\scriptsize{$\pm$6.9}        & 94.1\scriptsize{$\pm$1.0}                   \\
MixDiff+ZOC $^\dagger$      & \ding{55}     & \textbf{82.8}\scriptsize{$\pm$2.4}    & 65.2\scriptsize{$\pm$12.0}               & \textbf{95.0}\scriptsize{$\pm$0.7}                \\ \midrule \midrule
MSP                         & \ding{51}     & 78.2\scriptsize{$\pm$3.1}            & 60.4\scriptsize{$\pm$5.3}                & 91.4\scriptsize{$\pm$1.9}                   \\
MLS                         & \ding{51}     & 80.0\scriptsize{$\pm$3.1}             & 62.3\scriptsize{$\pm$5.2}                & 92.9\scriptsize{$\pm$1.6}                 \\
Energy                      & \ding{51}     & 77.6\scriptsize{$\pm$3.7}             & 65.4\scriptsize{$\pm$4.2}                & 91.9\scriptsize{$\pm$1.9}                  \\
Entropy                     & \ding{51}     & 79.9\scriptsize{$\pm$2.5}             & \underline{58.8}\scriptsize{$\pm$5.2}  & 92.0\scriptsize{$\pm$1.7}                  \\ \midrule
MixDiff+MSP                 & \ding{51}     & 80.1\scriptsize{$\pm$2.8}            & 60.1\scriptsize{$\pm$4.8}               & 92.3\scriptsize{$\pm$1.5}                 \\
MixDiff+MLS                 & \ding{51}     & \underline{80.5}\scriptsize{$\pm$2.2}        &62.5\scriptsize{$\pm$4.1}      & \textbf{92.9}\scriptsize{$\pm$1.2}                  \\
MixDiff+Energy              & \ding{51}     & 78.3\scriptsize{$\pm$2.7}              & 65.9\scriptsize{$\pm$3.4}              & 92.1\scriptsize{$\pm$1.4}                  \\
MixDiff+Entropy             & \ding{51}     & \textbf{81.0}\scriptsize{$\pm$2.6}     & \textbf{58.4}\scriptsize{$\pm$4.8}    & \underline {92.6}\scriptsize{$\pm$1.5} \\ \bottomrule
\end{tabular}
}
\caption{Performance comparison with various metrics.}
\label{tab:oth_metrics}
\end{table}

\section{Comparison with other OOD scoring functions}
We divide the MSP score into five intervals of the same length and plot the difference of the average scores of OOD and ID samples in the same interval. We also plot the difference of the average MixDiff scores of OOD and ID samples belonging to the same MSP score interval. Figure \mbox{\ref{fig:ood_id_diff}} shows that for similar values of MSP score, the uncertainty score from MixDiff among the OOD samples is significantly higher than that of the ID samples. This demonstrates that, even when two ID, OOD samples' MSP scores are almost identical, the MixDiff scores can still provide a discriminative edge.

\begin{figure}[h]
  \centering
  \begin{subfigure}[b]{0.48\linewidth}
    \centering
    \includegraphics[width=0.9\linewidth]{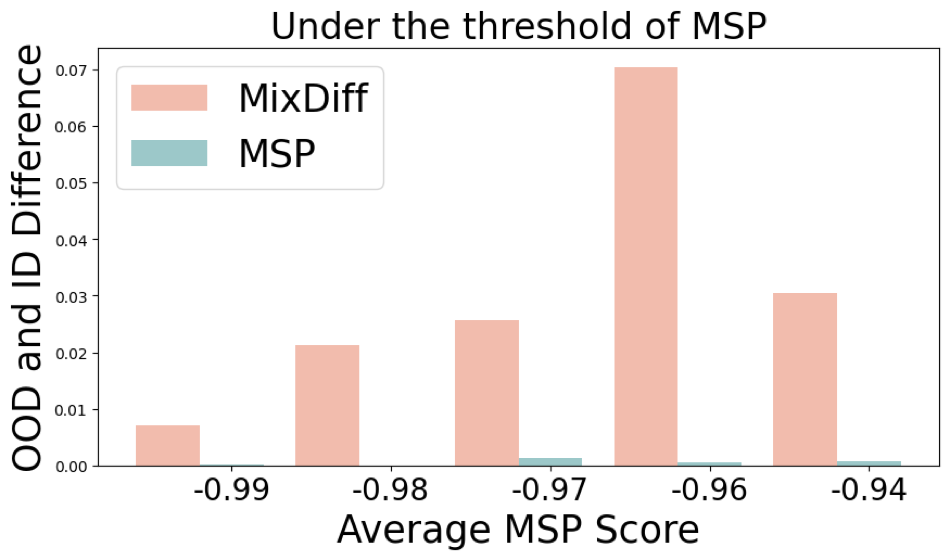} 
    \caption{} 
    \label{fig:interval_over} 
  \end{subfigure}
  \begin{subfigure}[b]{0.48\linewidth}
    \centering
    \includegraphics[width=0.9\linewidth, valign=t]{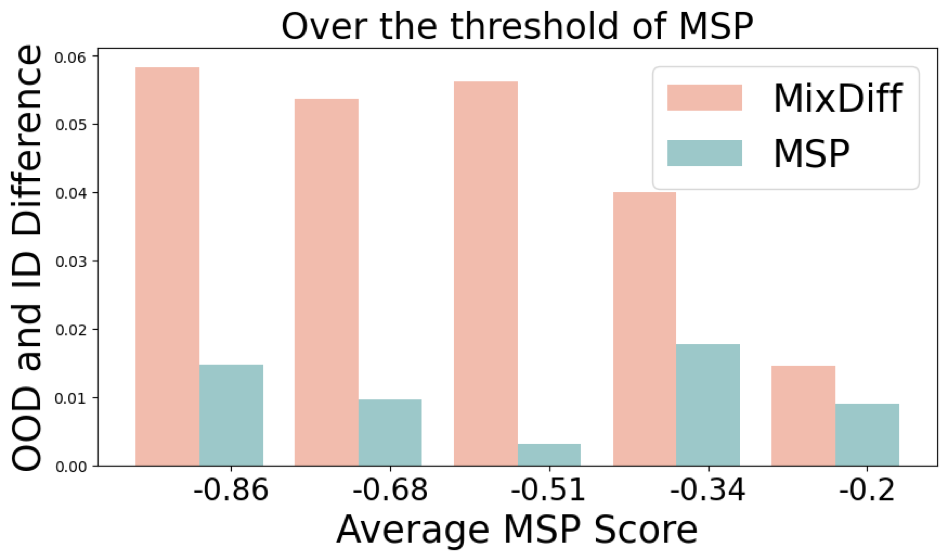} 
    \caption{} 
    \label{fig:interval_under} 
  \end{subfigure}
\caption{Difference between the average uncertainty scores of OOD and ID samples belonging to a given interval of MSP score. The $x$-axis represents the average MSP score of the interval. \textbf{(a)} Intervals under the threshold with the threshold set by TPR95 of MSP. \textbf{(a)} Intervals over the threshold with the threshold set by TPR95 of MSP.}
\label{fig:ood_id_diff}
\end{figure}

To validate whether MixDiff scores have extra information which is not captured by existing other OOD scores, we calculate pair-wise correlation among OOD scores in Figure \ref{fig:interval_over}, and evaluate the error rate of OOD detection by each OOD score in Figure \ref{fig:interval_under}. 

\begin{figure}[h]
  \centering
  \begin{subfigure}[b]{0.4\linewidth}
    \centering
    \includegraphics[width=0.9\linewidth]{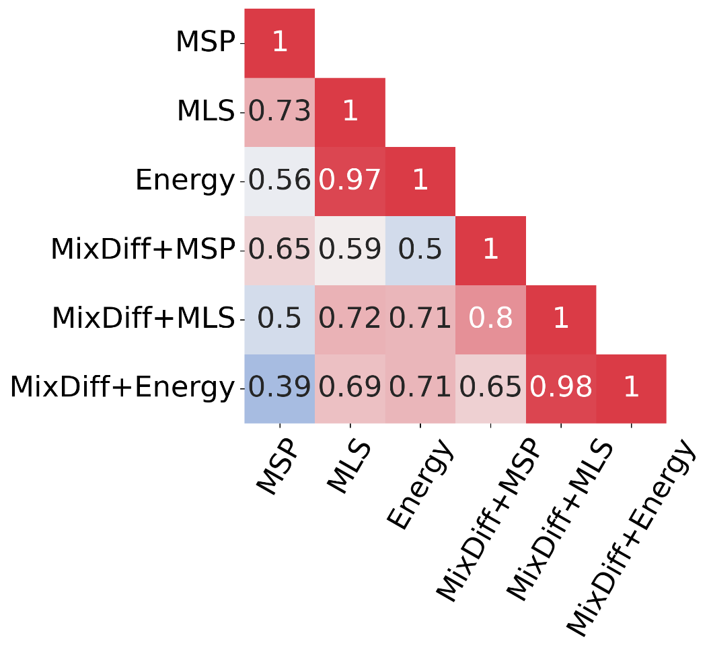} 
    \caption{} 
    \label{fig:qa_figures_a} 
  \end{subfigure}
  \begin{subfigure}[b]{0.4\linewidth}
    \centering
    \includegraphics[width=0.9\linewidth, valign=t]{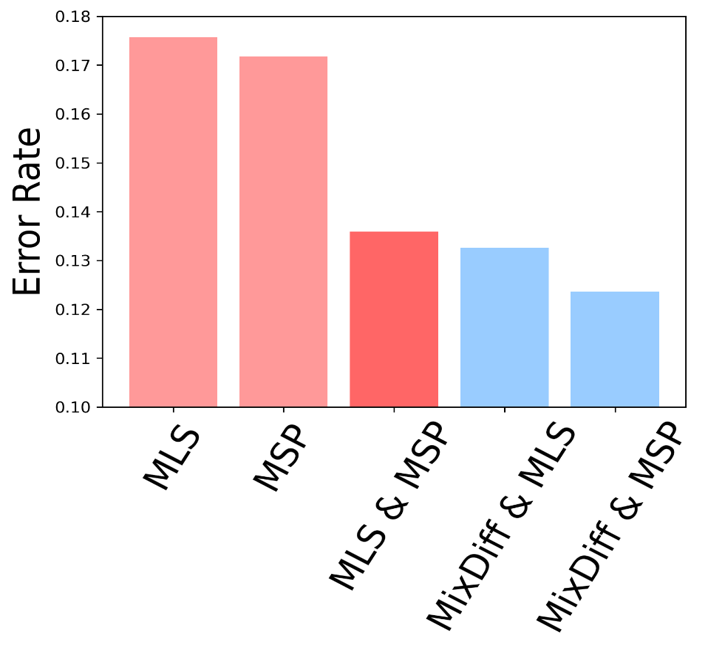} 
    \caption{} 
    \label{fig:qa_figures_b} 
  \end{subfigure}
  \caption{\textbf{(a)} Pearson correlation between the scores of different OOD scoring functions. \textbf{(b)} Error rate at TPR95 for each method. For multiple methods, error means both were incorrect.}
  \label{fig:qa_figures}
\end{figure}

As shown in Figure \ref{fig:qa_figures_a}, MixDiff scores exhibit a weaker correlation with other OOD scores, which implies that MixDiff scores contain additional information that is absent in other scores. Consequently, MixDiff can correct certain wrong decisions of existing methods (verified in Figure \ref{fig:qa_figures_b}), when adopted with them together. These results suggest that the perturb-and-compare approach is helpful for stable OOD detection and MixDiff effectively provides such an advantage. All results from this subsection are derived from CIFAR100 test set.

\section{Computational cost analysis}
Figure \mbox{\ref{fig:auroc_by_rn}} shows AUROC scores of MixDiff+Entropy for various values $R$ and $N$ evaluated on CIFAR100. MixDiff starts to outperform the entropy score with only two additional forward passes ($N=2$, $R=1$). The model outputs from $f(\cdot)$ are prediction probabilities, the number of oracle samples, $M$, is fixed at 15 and the scaling factor $\gamma$ is tuned on Caltech101.

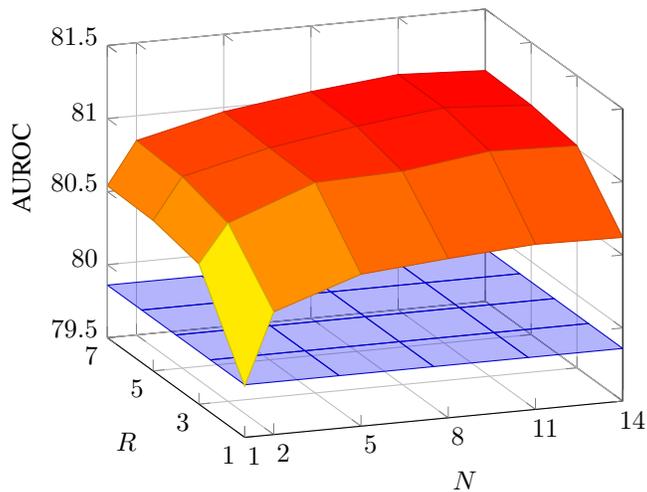
\begin{figure}[h]
\centering
\centering
\begin{tikzpicture}
    \begin{axis}[
        view={-20}{20}, grid=both,
        xlabel = $N$,
        ylabel = $R$,
        zlabel = AUROC,
        xtick = {1,2,5,8,11,14},
        ytick ={1,3,5,7},
        zmin = 79.5, zmax = 81.5
    ]
      
      \addplot3 [
          surf,
          samples=5,
          domain=1:14,
          domain y=1:7,
          fill opacity = 0.3,
      ]
      (x,y,79.86);

      \addplot3[
      surf,
      ]
      file {filename.txt};
    \end{axis}
\end{tikzpicture}
\caption{AUROC scores of MixDiff+Entropy with varying values of $N$ and $R$ (top). AUROC score of Entropy (bottom). Both methods are evaluated on CIFAR100.}
\label{fig:auroc_by_rn}
\end{figure}

\section{Processing time analysis}
We analyze the average time required to process one target sample. Target-side perturbed samples processed in a single batch. We fix the number of oracle samples, $M$, to 15 and use the other oracle samples as the auxiliary samples ($N$=14). This is the same as the oracle as auxiliary setup in the ablation studies portion of the main paper. We precompute the oracle-side perturbed samples. Figure \ref{fig:proc_time} depicts the average processing time against the number of Mixup ratios, $R$. The stagnant increase in processing time contrasted with the rapid increase in performance at small values of $R$ indicates that the additional perturbed samples can be effectively processed in parallel, so that MixDiff's effectiveness can be exploited without incurring a prohibitive processing time. When we allow multiple target samples to be batched together, MixDiff's processing time further decreases (MixDiff BS=100). Experiments are performed with NVIDIA RTX A6000 48GB.

\begin{figure}[h]
\centering
    \begin{tikzpicture}
    \begin{axis}[
        xlabel={Number of Mixup Ratios (R)},
        ylabel={Processing Time (Milliseconds)},
        ymax=45,
        legend pos=north west,
        ymajorgrids=true,
        grid style=dashed,
    ]
    
    \addplot[
        color=blue,
        mark=*,
        ]
        coordinates {
        (1,14.511)(2,15.062)(3,18.846)(4,26.655)(5,30.569)(6,32.974)(7,37.464)
        };
        \addlegendentry{MixDiff BS=1}
        
    \addplot[
        color=blue,
        mark=square,
        ]
        coordinates {
        (1,3.886)
        (2,7.184)
        (3,10.426)
        (4,13.721)
        (5,16.981)
        (6,20.225)
        (7,23.444)
        };
        \addlegendentry{MixDiff BS=100}
    
      \addplot[mark=none, blue] coordinates {(1,7.117) (7,7.117)};
      \addlegendentry{Entropy BS=1}
    \end{axis}
    
    \begin{axis}[
      axis y line=right,
      axis x line=none,
      ymin=79.5, ymax=82,
      ylabel=AUROC
    ]
    \addplot[smooth,mark=*,red] 
      coordinates{
        (1,80.2)
        (2,80.85) 
        (3,81.01)
        (4,81.06)
        (5,81.09) 
        (6,81.1) 
        (7,81.11)
    }; \addlegendentry{W/ MixDiff}
    
      \addplot[mark=none, red] coordinates {(1,79.87) (7,79.87)};
      \addlegendentry{W/o MixDiff}
    
    \end{axis}
    \end{tikzpicture}
\caption{Blue lines represent the average processing time per target sample. BS denotes the batch size of target samples. Red lines represent AUROC scores of MixDiff+Entropy and entropy OOD scoring function evaluated on CIFAR100.}
\label{fig:proc_time}
\end{figure}
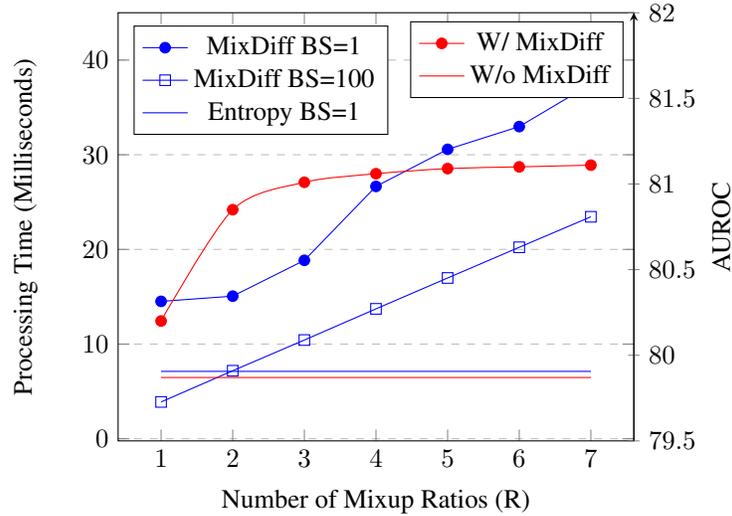

\section{Sensitivity analysis}

Figures \ref{fig:sensitivity_m} and \ref{fig:sensitivity_r} 
show the changes in AUROC score on the CIFAR100 dataset in regard to the number of oracle samples, $M$, and the number of Mixup ratios, $R$, respectively. We fix the other hyperparameters and only vary $M$ or $R$. MixDiff starts to enhance the detection performance of base scores with small values of $M$ or $R$, after which the performance gain remains relatively stable. For all OOD scoring functions, logits are used as the model $f(\cdot)$'s outputs when computing perturbed oracles' OOD scores. 

\begin{figure}[H]
\centering
\resizebox{0.7\columnwidth}{!}{
\centering
    \begin{subfigure}[t]{0.35\linewidth}
        \centering
        \includegraphics[width=\linewidth]{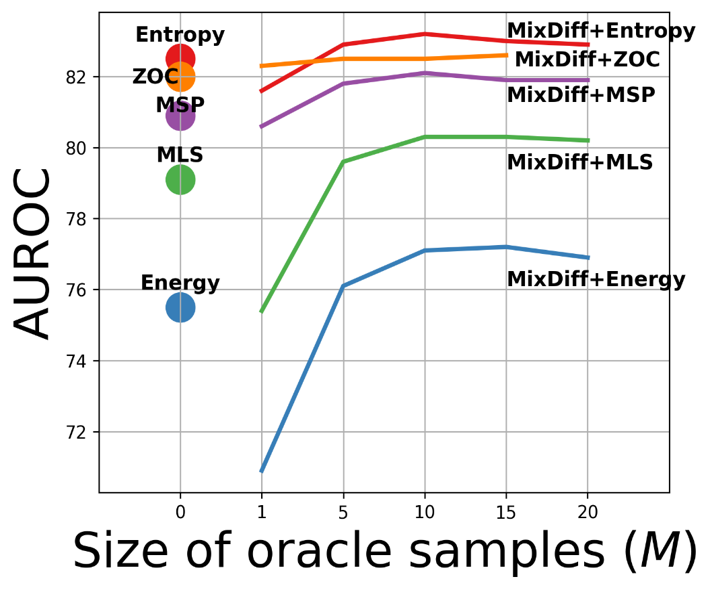}
        \caption{} \label{fig:sensitivity_m}
    \end{subfigure}\hfill
    \begin{subfigure}[t]{0.35\linewidth}
        \centering
        \includegraphics[width=\linewidth]{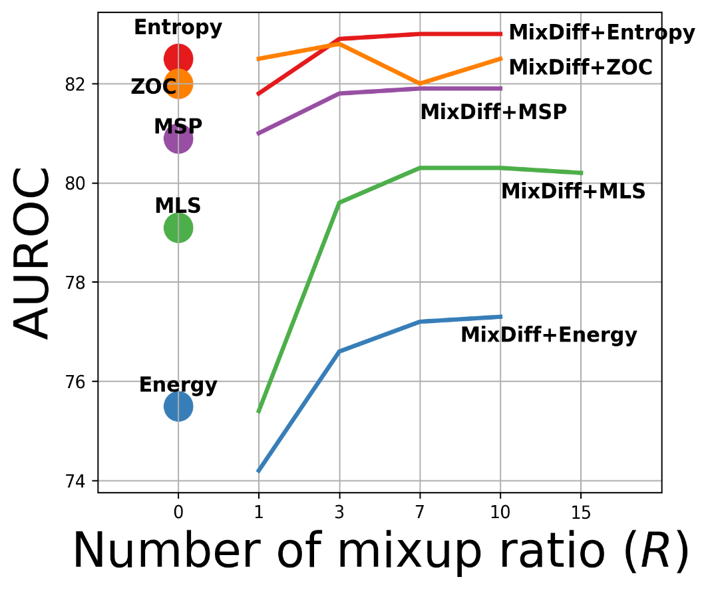}
        \caption{} \label{fig:sensitivity_r}
    \end{subfigure}\hfill
}
\caption{\textbf{(a)} Performance change in regard to the number of oracle samples, $M$. \textbf{{(b)}} Performance change in regard to the number of Mixup ratios, $R$.}
\label{fig:sensitivity}
\end{figure}

\section{Performance evaluation with other backbones}

We evaluate MixDiff's performance with various CLIP backbones and report the results in Table \ref{tab:diff_backbones}. We also report the classification accuracy of each classifier on the ID test set. MixDiff consistently improves the detection performance of the base score. Average AUROC scores over the five datasets are reported, and oracle as auxiliary setup is used for auxiliary sample selection. Prediction probabilities are used as model outputs.

\begin{table}[H]
\centering
\resizebox{0.7\columnwidth}{!}
{%
\begin{tabular}{lccc} \toprule
\multirow{2}{*}{CLIP Backbone} & \multicolumn{2}{c}{Average AUROC} &  \multirow{2}{*}{Classification Acc.} \\  \cmidrule{2-3}
& Entropy & MixDiff+Entropy & \\ \midrule
RN50 & 80.17 & 81.05 & 73.36  \\ 
RN50x4 & 82.12 & 83.23 & 78.20 \\ 
VIT-B/32 & 89.11 & 89.88 & 87.73 \\ 
ViT-L/14 & 93.40 & 94.19 & 92.59 \\ \bottomrule 
\end{tabular}
}
\caption{Performance evaluation with various CLIP backbones. }
\label{tab:diff_backbones}
\end{table}

\section{Performance evaluation under varying misclassification rates}
We construct the ID sets such that it would contain a specific percentage of misclassified samples and evaluate MixDiff's performance on various misclassification rates. Table \ref{tab:miscl} shows that MixDiff exhibits significant improvements over the baseline when the percentage of misclassified samples is high. Average AUROC scores over the five datasets are reported, and random ID samples are used as auxiliary samples. Prediction probabilities are used as model outputs.

\begin{table}[H]
\centering
\resizebox{0.5\columnwidth}{!}
{%
\begin{tabular}{ccc} \toprule
Misclassification rate & Entropy & MixDiff+Entropy \\ \toprule
100\% & 65.24 & 72.06  \\
75\% & 72.10 & 77.76 \\
50\% & 78.98 & 83.50  \\
25\% & 85.70 & 89.09  \\
0\% & 89.11 & 89.69 \\ \bottomrule
\end{tabular}
}
\caption{Performance under varying misclassification rates. }
\label{tab:miscl}
\end{table}

\section{Reproducibility}
We make our code publicly available at  \href{https://github.com/hy18284/mixdiff}{https://github.com/hy18284/mixdiff}.

\section{Qualitative analysis}

\subsection{OOD score density curves}
Figure \ref{fig:dist} plots the distributions of the base OOD scores with and without MixDiff. Table \ref{tab:area_dist} shows the area under the distribution curves of in-distribution (ID) and out-of-distribution (OOD) samples separated by the threshold (set by FPR95) for each approach. MixDiff scores alleviate overlap of ID and OOD samples' OOD scores. In Table \ref{tab:area_dist}, we observe that adding MixDiff scores increases the area of the ID samples' distribution under the threshold and decreases the area of ID samples' distribution over the threshold. For all OOD scoring functions, logits are used as the model $f(\cdot)$'s outputs when computing perturbed oracles' OOD scores.

\begin{table}[H]
    \begin{minipage}{\linewidth}
        \centering
        \includegraphics[width=1.0\textwidth]{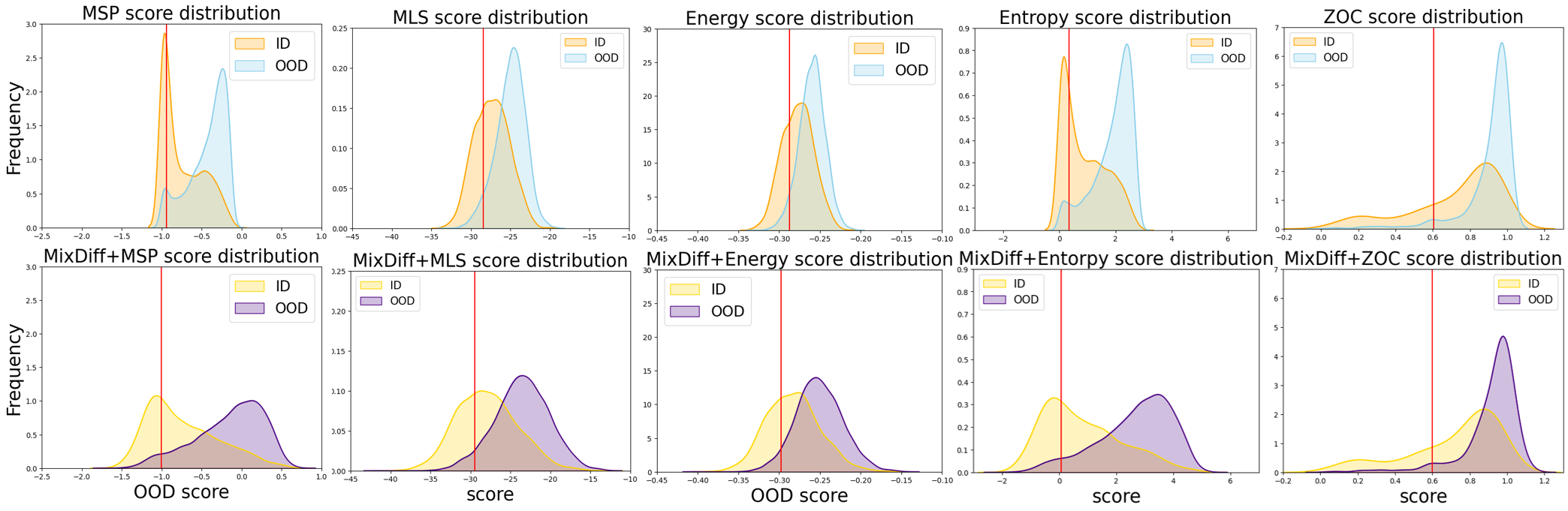}
        \captionof{figure}{Visualizations of distributions of the OOD scores with kernel density estimate plot. The red vertical lines represent 95$\%$ TPR thresholds.}
        \label{fig:dist}
    \end{minipage}\\
    \begin{minipage}{\linewidth}
        \centering
        \resizebox{\linewidth}{!}{%
        \begin{tabular}{lcccccccccc}
        \toprule
                                           & MSP    & MixDiff+MSP    & MLS    & MixDiff+MLS    & Energy & MixDiff+Energy & Entropy & MixDiff+Entropy & ZOC   & MixDiff+ZOC    \\ \midrule
        Threshold (95$\%$ TPR)            & -0.938 & -1.007         & -28.44 & -29.42         & -0.287 & -0.297         & 0.358   & 0.071           & 0.604 & 0.598          \\
        ID over threshold ($\downarrow$)   & 0.688  & \textbf{0.663} & 0.667  & \textbf{0.625} & 0.684  & \textbf{0.645} & 0.656   & \textbf{0.628}  & 0.731 & \textbf{0.725} \\
        ID under threshold ($\uparrow$)    & 0.296  & \textbf{0.322} & 0.322  & \textbf{0.360} & 0.304  & \textbf{0.340} & 0.331   & \textbf{0.358}  & 0.263 & \textbf{0.268} \\
        OOD over threshold    & 0.949  & {0.950} & 0.947  & 0.947          & 0.945  & {0.947} & 0.938   & {0.947}  & 0.947 & {0.949} \\
        OOD under threshold & 0.048  & {0.047} & 0.050  & {0.049} & 0.051  & {0.048} & 0.060   & {0.050}  & 0.051 & {0.049} \\ \bottomrule
        \end{tabular}%
        }
        \caption{The integral of density curves from Figure \ref{fig:dist} divided by 95\% TPR threshold. $\downarrow$ indicates lower is better and $\uparrow$ indicates higher is better.}
        \label{tab:area_dist}
    \end{minipage}
\end{table}

\subsection{Logit visualizations}
To see the effect of MixDiff in the logit level, we plot the logits of the target, oracle, and the corresponding mixed samples in Figure \ref{fig:logit_plot}. For all OOD scoring functions, logits are used as the model $f(\cdot)$'s outputs when computing perturbed oracles' OOD scores.

\begin{figure*}[ht]
\begin{subfigure}[b]{\linewidth}
    \centering
    \includegraphics[width=0.72\linewidth]{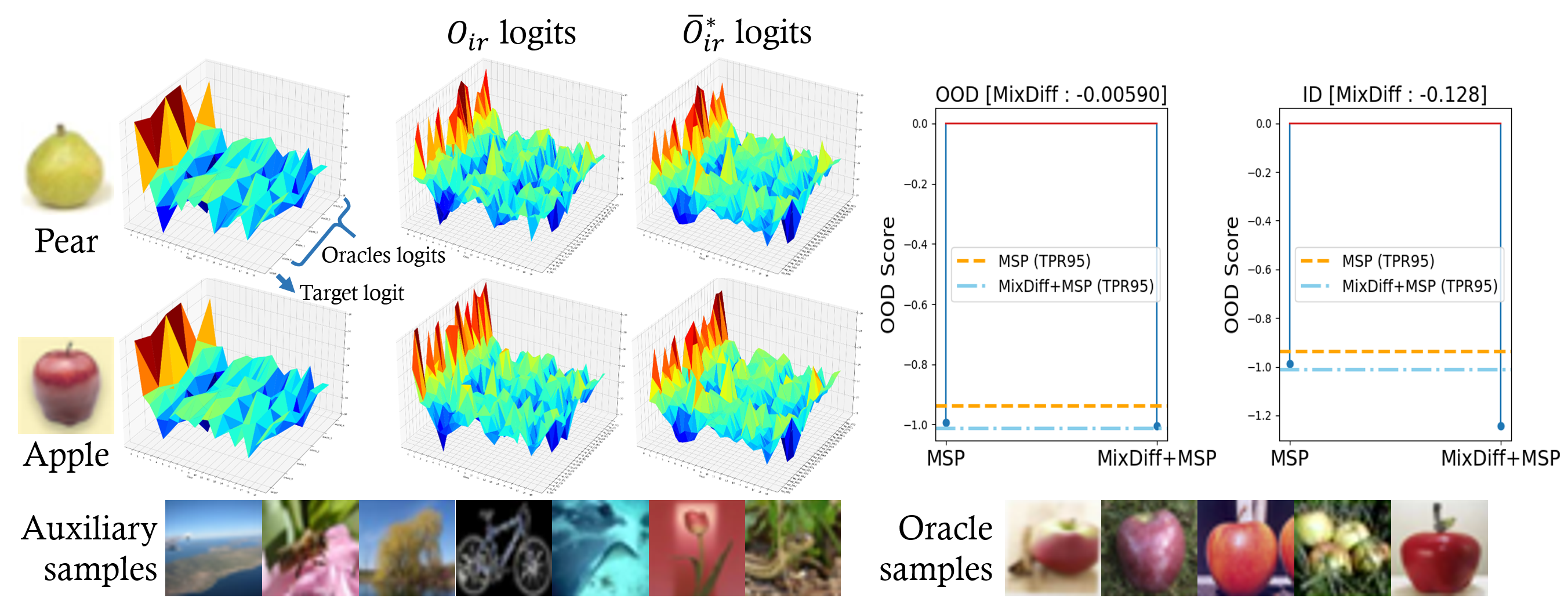}
    \caption{MixDiff+MSP\\}
    \label{fig:msp_example}
\end{subfigure}
\begin{subfigure}[b]{\linewidth}
    \centering
    \includegraphics[width=0.72\linewidth]{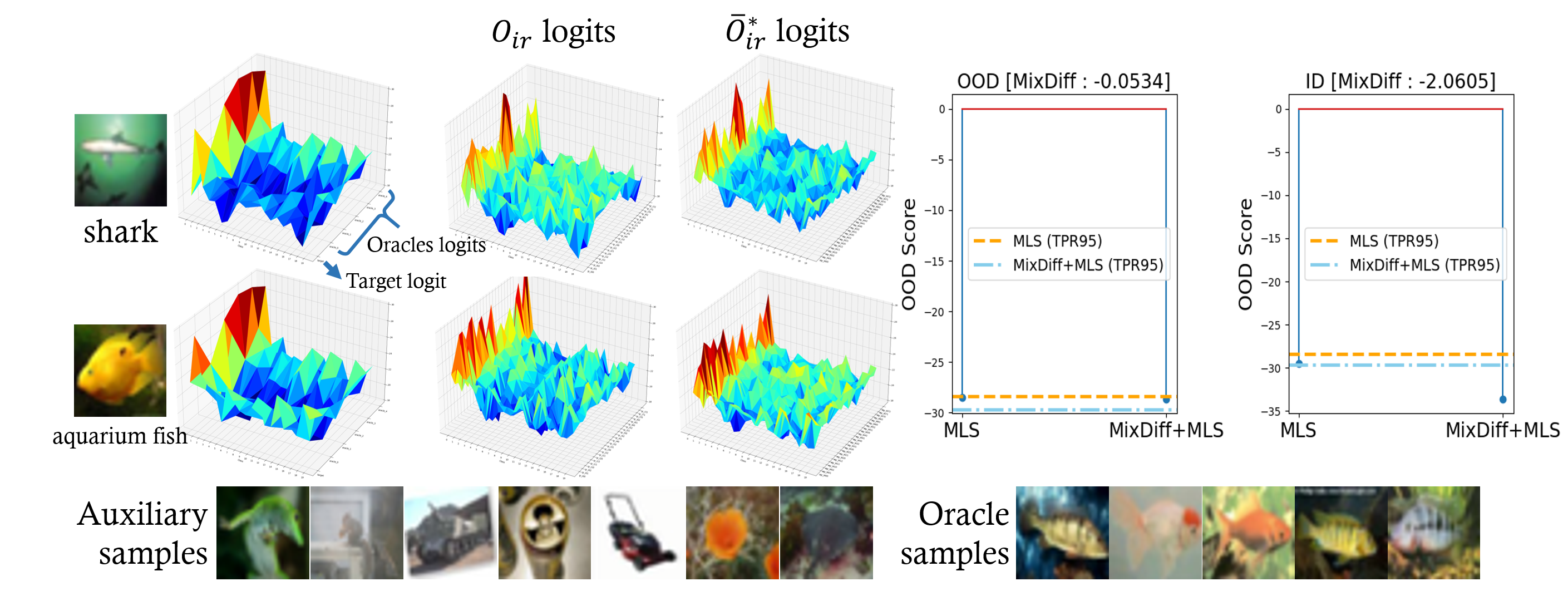}
    \caption{MixDiff+MLS\\}
    \label{fig:mls_example}
\end{subfigure}
\begin{subfigure}[b]{\linewidth}
    \centering
    \includegraphics[width=0.72\linewidth]{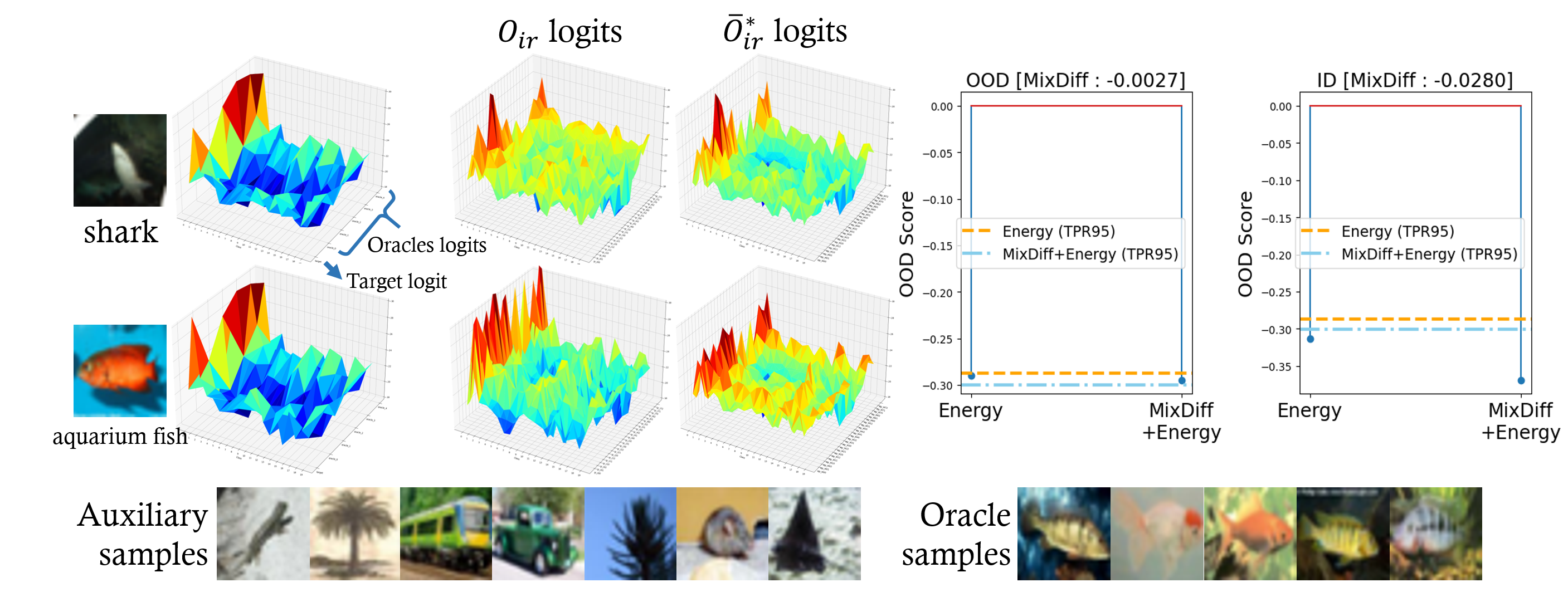}
    \caption{MixDiff+Energy\\}
    \label{fig:energy_example}
\end{subfigure}
\begin{subfigure}[b]{\linewidth}
    \centering
    \includegraphics[width=0.72\linewidth]{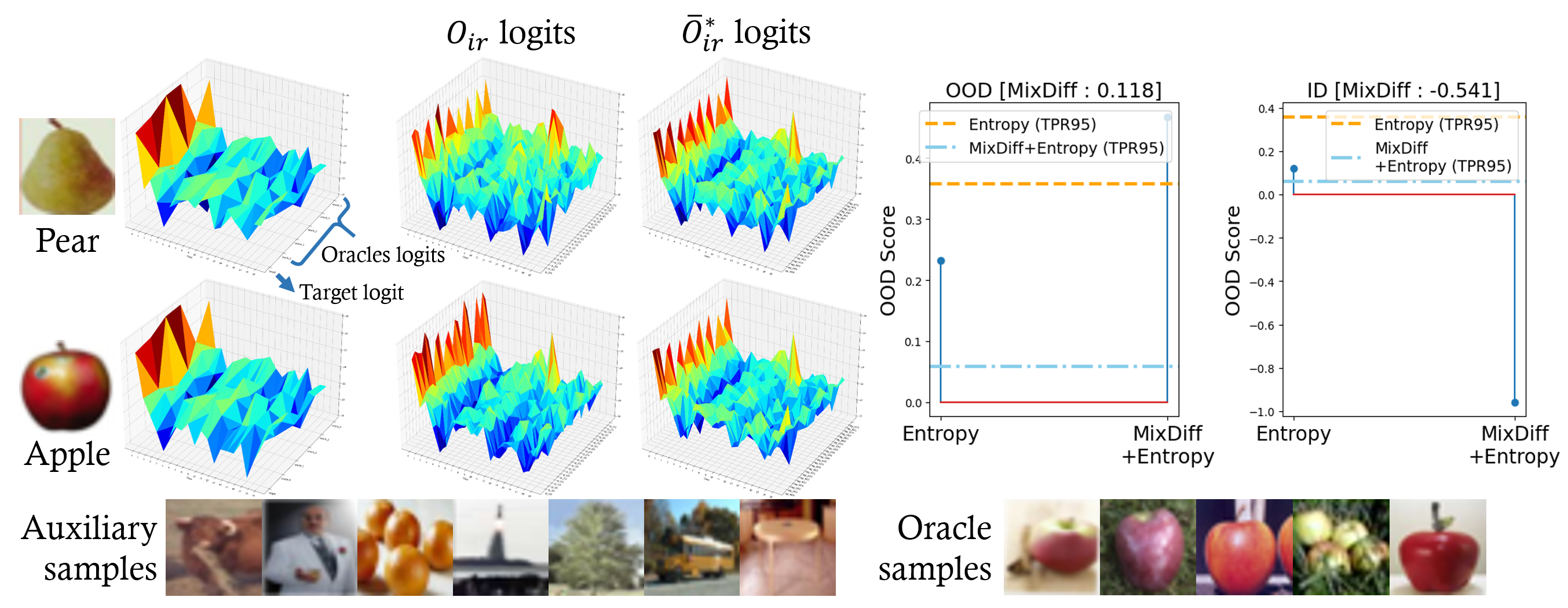}
    \caption{MixDiff+Entropy\\}
    \label{fig:entropy_example}
\end{subfigure}
\caption{Logit level changes after mixing identical auxiliary samples with target or oracle. The first row of logit graphs in Figures \ref{fig:msp_example}-\ref{fig:entropy_example} show that even though there is an OOD sample that is indistinguishable from the oracles at the logit level, the difference could be captured by mixing up with auxiliary samples. The the second row of 3D graphs in Figures \ref{fig:msp_example}-\ref{fig:entropy_example} show logits of the ID sample whose class is the same as the oracle samples. The two graphs to the right of each logit graph show the OOD scores and thresholds for the base OOD score function with and without MixDiff for the OOD and ID target samples, respectively.}
\label{fig:logit_plot}
\end{figure*}

\FloatBarrier

\bibliographystyle{named}
\bibliography{suppl}

%% file: math_commands.tex
\usepackage{amsmath,amsfonts,bm}

\def\eqref#1{equation~\ref{#1}}

\def\1{\bm{1}}

\DeclareMathAlphabet{\mathsfit}{\encodingdefault}{\sfdefault}{m}{sl}
\SetMathAlphabet{\mathsfit}{bold}{\encodingdefault}{\sfdefault}{bx}{n}

\DeclareMathOperator*{\argmax}{arg\,max}